\documentclass[colorlinks]{elsarticle}
\usepackage{amsmath}
\usepackage{amsthm}
\usepackage{amssymb}
\usepackage[numbers]{natbib}
\usepackage{tabularx}
\usepackage{algorithm}
\usepackage{algpseudocode} 
\usepackage[table]{xcolor} 
\usepackage{caption}
\usepackage{subcaption,graphicx}
\usepackage[export]{adjustbox}
\usepackage{lineno,hyperref}
\usepackage{varioref} 
\usepackage{refstyle}
\usepackage [english]{babel}
\usepackage [autostyle, english = american]{csquotes}
\MakeOuterQuote{"}
\allowdisplaybreaks

\modulolinenumbers[5]

\usepackage[colorinlistoftodos]{todonotes}

\newcommand{\rulesep}{\unskip\ \vrule\ }

\usepackage{changes}

\definecolor{pistachio}{rgb}{0.58, 0.77, 0.45}
\definecolor{cinnamon}{rgb}{0.82, 0.41, 0.12}
\definecolor{celadon}{rgb}{0.67, 0.88, 0.69}
\definecolor{cadmiumgreen}{rgb}{0.0, 0.42, 0.24}
\definecolor{amethyst}{rgb}{0.6, 0.4, 0.8}
\definechangesauthor[name=Ruslan Krenzler,color=cadmiumgreen]{RK}
\definechangesauthor[name=Lin Xie,color=amethyst]{LX}
\definechangesauthor[name=Nils Thieme, color=amethyst]{NT}

\theoremstyle{plain}

\newtheorem{prop}{Proposition}

\theoremstyle{remark}

\theoremstyle{definition}
\newtheorem{example}{Example}

\journal{EJOR}

\biboptions{authoryear}

\begin{document}

\begin{frontmatter}

\title{Efficient order picking methods in robotic mobile fulfillment systems}




\author[mymainaddress]{Lin Xie\corref{mycorrespondingauthor}}
\cortext[mycorrespondingauthor]{Lin Xie}
\ead{xie@leuphana.de}
\ead[url]{http://www.leuphana.de/or/}
\author[mymainaddress]{Nils Thieme}
\author[mymainaddress]{Ruslan Krenzler}
\author[mysecondaryaddress]{Hanyi Li}

\address[mymainaddress]{Leuphana University of L\"uneburg, Universit\"atallee 1, 21335 L\"uneburg}
\address[mysecondaryaddress]{Beijing Hanning Tech Co., Ltd}

\begin{abstract}
Robotic mobile fulfillment systems (RMFSs) are a new type of warehousing system, which has received more attention recently, due to increasing growth in the e-commerce sector. Instead of sending pickers to the inventory area to search for and pick the ordered items, robots carry shelves (called "pods") including ordered items from the inventory area to picking stations. In the picking stations, human pickers put ordered items into totes; then these items are transported by a conveyor to the packing stations. This type of warehousing system relieves the human pickers and improves the picking process. In this paper, we concentrate on decisions about the assignment of pods to stations and orders to stations to fulfill picking for each incoming customer's order. In previous research for an RMFS with multiple picking stations, these decisions are made sequentially. Instead, we present a new integrated model. To improve the system performance even more, we extend our model by splitting orders. This means parts of an order are allowed to be picked at different stations. To the best of the authors' knowledge, this is the first publication on split orders in an RMFS. We analyze different performance metrics, such as pile-on, pod-station visits, robot moving distance and order turn-over time. We compare the results of our models in different instances with the sequential method in our open-source simulation framework RAWSim-O.
\end{abstract}

\begin{keyword}
order picking problem \sep parts-to-picker warehouses \sep  robotic mobile fulfillment systems \sep  warehousing \sep simulation
\end{keyword}

\end{frontmatter}


\section{Introduction}\label{sec:intro}
The most important and time-consuming task in a warehouse is the collection of items from their storage locations to fufill customer orders. The process is called \textit{order picking}, which may constitute about 50--65\% of the operating costs. In a traditional manual order picking system (also called a \textit{picker-to-parts system}), the pickers  spend 70\% of their working time on the tasks of search and travel (see \cite{Tompkins.2010}; for an overview of manual order picking systems see \cite{de2007design}). Due to the increasingly fast-paced economy, it is becoming more and more important that the orders are processed in a short time window. Moreover, the order picking is considered as the highest-priority area for productivity improvements (see \cite{de2007design}). 

The picker-to-parts system is no longer fit for e-commerce operations, since the companies usually have millions of small items in large warehouses. Kiva Systems LLC, now Amazon Robotics LLC, came up with a unique solution that accelerates the order picking process (see \cite{Wurman.2008}). Robots are sent to carry storage units, so-called "pods," from the inventory area and bring them to human operators, who work at picking stations. At the stations, the items are picked according to the customers' orders. This system is called \textit{robotic mobile fulfillment system} (RMFS). There are also some other suppliers of such systems, such as Scallog, Swisslog (KUKA),  GreyOrange and Hitachi (see \cite{banker2016robots}). All of these systems may differ technically in certain aspects, such as the lifting mechanism, but they share the same principle of the system (see Section \ref{description}).

There are four decision problems for each new incoming order in an RMFS. We have to decide which robot carries which pod along which path to which station to fulfill orders. We concentrate in this paper on decisions about the assignment of orders to stations (pick order assignment, in short: POA) and pods to stations (pick pod selection, in short: PPS). The other two decisions (robot task allocation, path finding) are made by the existing methods in our open-source simulation "RAWSim-O" (see \cite{RawSim.1}). 

We consider an efficient order picking system as a system to handle more orders within minimal time (as suggested in \cite{van2018designing} for picker-to-parts systems). In an efficient RMFS, we try to keep pickers busy by providing pods to pick from. In order to achieve that, we want to reduce idle time between changes of pods, so in this paper we aim at minimizing the visits by pods to stations for given sets of orders for the POA and PPS problems.  We will explain more about the objective in Section~\ref{subsec:obj}. In the following subsections, we describe two contributions of this paper. There are two methods to achieve our goal.

\subsection{Contribution I: Integrated POA and PPS for multiple stations}
In the literature, POA and PPS are usually solved sequentially for multiple stations (first POA, then PPS; see \cite{Wurman.2008}, \cite{RawSimDecisionRules}). Each time a new order arrives (and there is at least one free space in a picking station), it is first assigned to a station (POA) and then one or several pods are assigned to that station to fulfill that order (PPS). In this work, we integrate both the POA and PPS problems, and we will explain in Section~\ref{subsec:example} the differences between the sequential and integrated problems and also the possible advantages of using the integrated approach. \cite{Boysen.2017} provide methods for optimally batching the orders and sequencing both the orders and the pods at a single picking station. They prove that these problems are NP-hard. As an RMFS usually has more than one picking station, the decisions made at one station may influence decisions at other stations. Therefore, the decisions for several picking stations cannot be calculated separately. Our integrated approach allows for more than one picking station. This increases the complexity of both problems. Furthermore, we might encounter synchronization problems between several stations, for example the delayed arrival of a planned pod coming from another station (see a similar problem description in \cite{krenzler2018deterministic}). Moreover, instead of optimally batching the orders as in \cite{Boysen.2017}, in the real world we have to make decisions for both problems as soon as some jobs are finished at the stations, while there are unfulfilled orders. We call them \textit{periodic decisions}. Also, situations such as the inventory of pods and the positions of pods in the queues at stations can and will change over time. They are important for the POA and PPS decisions and are hard to calculate exactly in advance, since errors and delays in previous time periods can affect them. For these reasons, instead of calculating the optimal assignments for all time periods in advance (as in \cite{Boysen.2017}), we make the decisions for the integrated POA and PPS right before the respective time period starts. This allows us to react to the current situation and take errors or delays from previous periods into account. Furthermore, we test our results in a simulation framework, which provides us with the actual information for each time period.

\subsection{Contribution II: Allowing split orders in an RMFS}

In our integrated approach mentioned above, an order is only allowed to be assigned to a single station. The second contribution of this paper is to allow split orders in our integrated approach. A \textit{split order} means that we divide an order into two or more parts for picking (perhaps at different stations). A similar term, ``splitting orders,'' can be traced back to 1979, when it was used by \cite{Armstron.1979}. They used split orders to keep batch sizes constant in batch picking. In \cite{overview.old} and \cite{Koster.zones} split orders are used as part of the zoning in traditional picker-to-parts warehouses, in which a storage area is split into multiple parts (called zones), each with a different order picker. When an order contains several SKUs (stock keeping units) that are stored in different zones, the SKUs for the order are picked separately in each zone and merged later for shipping. To the best of the authors' knowledge, this is the first publication on split orders in an RMFS. 

According to the following example, we expect allowing split orders in an RMFS provide a better solution. 

\begin{example}
Figure \ref{fig:example_splitting} illustrates the decision problem when assigning orders and pods to two picking stations. We have one empty tote at each station, while we have two identical orders 1 and 2 in the backlog (in Figure~\ref{fig:t_0_split}). We assume that each tote can hold two items. These two orders contain SKUs shown in blue and orange. These two SKUs are located in two different pods, namely pod 1 with the orange SKU and pod 2 with the blue SKU. Figure \ref{fig:t_1_without_split} shows the optimal solution to the problem without split orders. We need pod 1 to visit station 1 and pod 2 to visit station 2; after that, pod 2 visits station 1 and pod 1 visits station 2. In total, we need four visits by pods to the stations to fulfill both orders. Instead, if we split orders 1 and 2 into blue and orange parts (see Figure~\ref{fig:t_1_with_split}), the blue ones can be picked from pod 2 at station 2, while the orange ones can be picked from pod 1 at station 1. This allows both orders to be fulfilled with only two visits by pods to the stations instead of four. Note that in this paper we don't use one empty tote for exactly one order, but for several items, to enable comparison between the solutions with and without split orders. We will explain more about this in Section~\ref{subsec:assumption} in the paragraph \textit{Capacity of a picking station}. 
\end{example}

\begin{figure}[h]
	\begin{subfigure}[t]{\textwidth}
		\hspace{+0.5cm}
		\includegraphics[width = 0.2\textwidth, left]{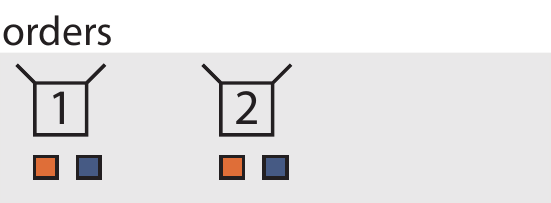}
	\end{subfigure}
	\newline
	\centering	
	\begin{subfigure}[b]{0.3\textwidth}
		\includegraphics[width = \textwidth]{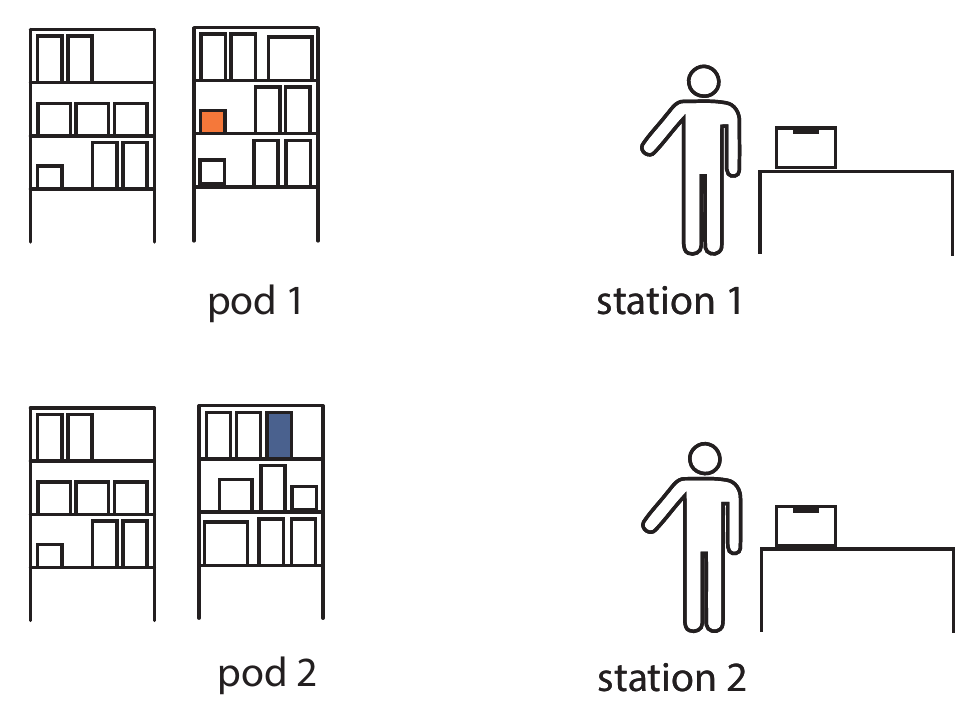}
		\newline \newline \newline \newline \newline \newline \newline \newline 
		\caption{Initial state \newline \newline \newline \newline \newline}	
		\label{fig:t_0_split}
	\end{subfigure}
	\rulesep
	\begin{subfigure}[b]{0.3\textwidth}
		\includegraphics[width = \textwidth]{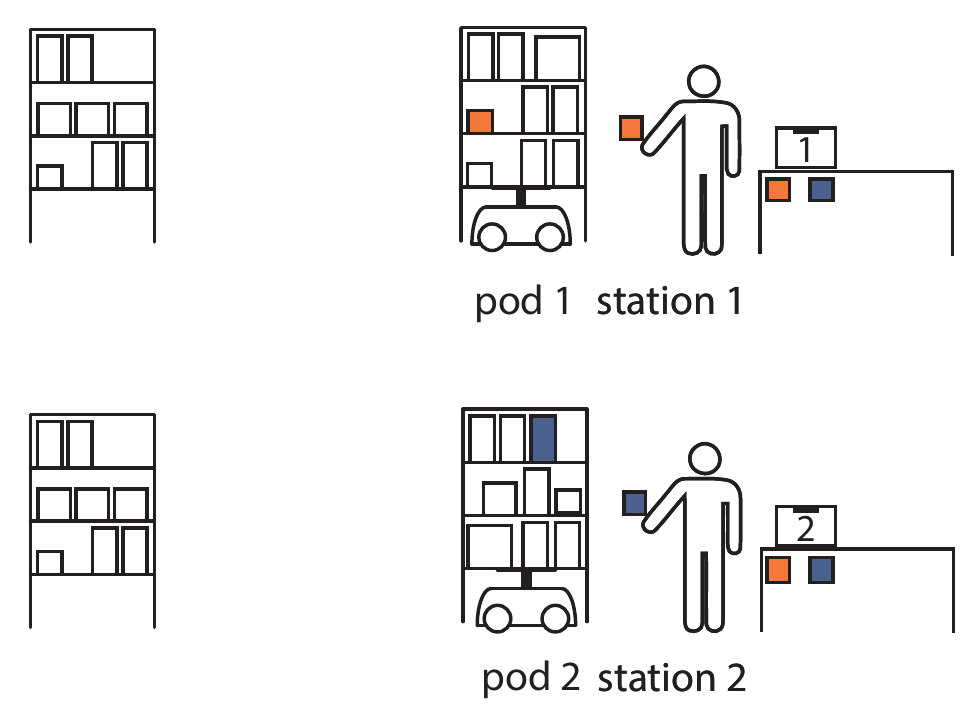}
		\noindent\rule{\textwidth}{0.2pt}
		\includegraphics[width = \textwidth]{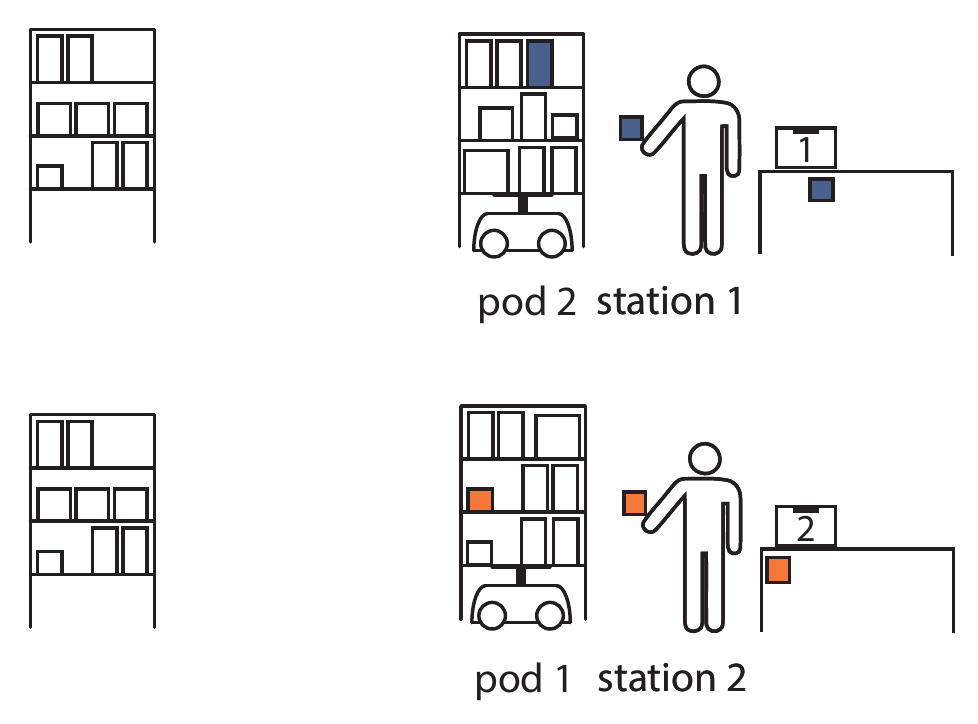}
		\caption{Without split orders: first, pod 1 $\rightarrow$ station 1 and pod 2 $\rightarrow$ station 2 (upper part); then, pod 2 $\rightarrow$ station 1 and pod 1 $\rightarrow$ station 2 (lower part)}	
		\label{fig:t_1_without_split}
	\end{subfigure}
	\rulesep
	\begin{subfigure}[b]{0.3\textwidth}
		\includegraphics[width = \textwidth]{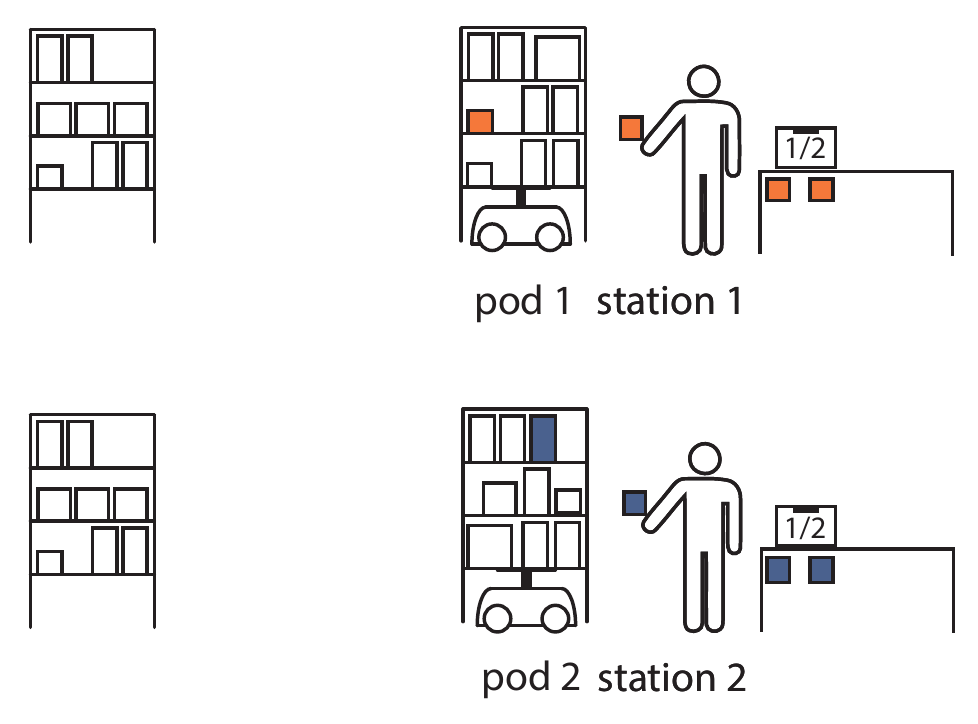}
		\newline \newline \newline \newline \newline \newline \newline \newline 
		\caption{With split orders: pod 1 $\rightarrow$ station 1 and pod 2 $\rightarrow$ station 2 \newline \newline \newline} 
		\label{fig:t_1_with_split}
	\end{subfigure}
	\caption{A solution for without and with split orders.}
	\label{fig:example_splitting}
\end{figure}

However, the split orders might cause additional waiting time for consolidation in packing stations, where customers' orders are packed and ready for shipping. Therefore, we analyze the turn-over times of orders -- the time from when the order arrives to when it is completely picked -- in Section~\ref{sec:results}. 
\subsection{Paper structure}
This paper is organized as follows: In the next section, we describe the RMFS and operational decision problems in detail. In Section~\ref{sec:model}, a mathematical model of integrated POA and PPS and the extensions with split orders are described. We present simulation evaluations in Section~\ref{sec:results}. Finally, we draw conclusions and give pointers for further research in Section~\ref{sec:conclusions}.

\section{Problem description} \label{description}
In this section, we first describe the RMFS, and the decision problems in an RMFS. After that, we explain the objective of integrated POA and PPS in this paper. Finally, we explain the difference between sequential and integrated POA and PPS with an example. 

\subsection{RMFS} \label{subsec:rmfs}
Firstly, we need to define some terms related to orders before explaining the processes in an RMFS, as follows:
\begin{itemize}
	\item stock keeping unit (\textit{SKU})
	\item an \textit{order line} consists of one SKU with the ordered quantity 
	\item an \textit{item} is a physical unit of one SKU
	\item a \textit{pick order} includes a set of order lines from a customer's order
	\item a \textit{split order} is a pick order that is separated into several parts
	\item a \textit{replenishment order} consists of a number of physical units of one SKU
	\item a \textit{backlog} includes all unfulfilled orders.
\end{itemize}
The central components of an RMFS are:
\begin{itemize}
	\item movable shelves, called \textit{pods}, on which inventory is stored
	\item \textit{storage area} denoting the inventory area where the pods are stored
	\item workstations, where 
	\begin{itemize}
		\item the pick order items are picked from pods by pickers (\textit{picking stations}) or
		\item the pick order items are packed by packers and the split orders are consolidated (\textit{packing stations}) or
		\item replenishment orders are stored to pods (\textit{replenishment stations})
	\end{itemize}
	\item mobile \textit{robots}, which can move underneath pods and carry them to workstations
	\item \textit{conveyors} between picking and packing stations to transport the pick orders to be packed.
\end{itemize}

The pods are transported by robots between the inventory area and workstations. Figure \ref{fig:storage_retrieval_process} shows the central process in an RMFS from replenishment to packing: 
\begin{itemize}
	\item \textit{Retrieval process}: After the arrival of a replenishment order, robots carry selected pods to a replenishment station to store units in pods. Similarly, after receiving a pick order, robots carry selected pods to a picking station, where the items for the order lines are picked. Note that in order to fulfill pick orders, several pods may be needed, since orders may have multiple lines. The items in (parts of) an order are picked into a tote.
	\item \textit{Storage process}: After a pod has been processed at one or more stations, it is brought back to a storage location in the storage area. The retrieval and storage processes are based on \cite{Hoffman.2013}.
	\item \textit{Transport to packing stations}: Once a tote is filled, it is transported by a conveyor to packing stations for packing.
	\item \textit{Packing process}: If all items in an order are contained in a tote, packers are prompted by computer to select the correct-sized box and pack the items. A split order has items delivered via multiple totes, since the items are picked by different pickers (picking stations). In this case, packers first sort items from a tote to a correct-sized box on the shelf so that the items from that order are grouped together. We use the term \textit{shelf} to clarify that they might be different to the pods, since they don't need to be moved. Once all the items of a split order are in a box, the packer packs the box, and a space is open for the next split order. This packing process is based on the packing process in Amazon (see \cite{website:packing}). 
\end{itemize}   
\begin{figure}
	\centering
	\includegraphics[width=\textwidth]{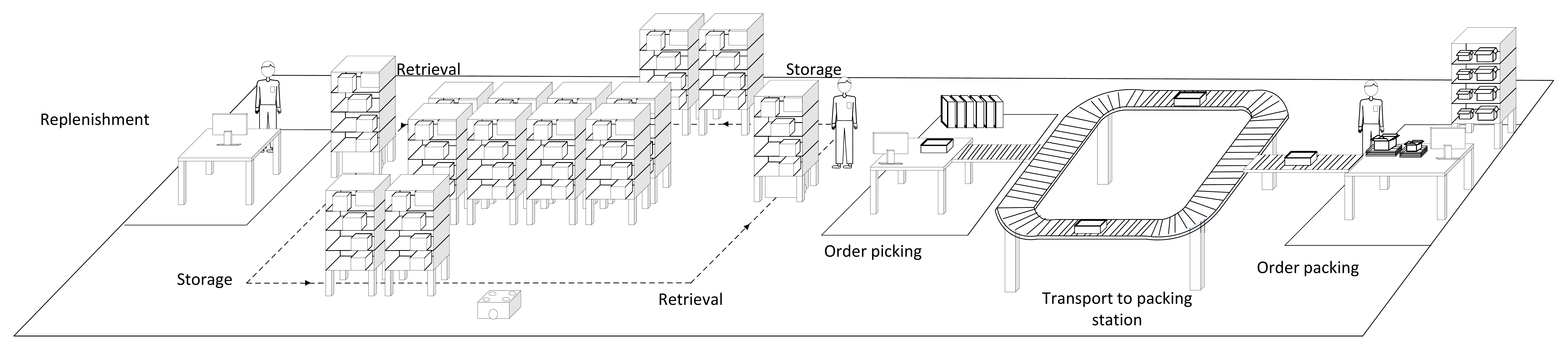}
	\caption{The central process of an RMFS.} 
	\label{fig:storage_retrieval_process}
\end{figure}
\subsection{Operational decision problems}
In an RMFS environment, various optimization and allocation problems have to be solved in real time. Note that there are also planning problems in an RMFS (see \cite{azadeh2017robotized} for an overview). The following operational decision problems were described in \cite{Wurman.2008} and \cite{RawSim.1}:

\begin{itemize}
	\item \textit{Order Assignment}
	\begin{itemize}
		\item \textbf{Replenishment Order Assignment (ROA)}: assignment of replenishment orders to replenishment stations
		\item \textbf{Pick Order Assignment (POA)}: assignment of pick orders to picking stations
	\end{itemize}
	\item \textit{Task Creation}
	\begin{itemize}
		\item \textit{Pod Selection}
		\begin{itemize}
			\item \textbf{Replenishment Pod Selection (RPS)}: selection of pods to store a replenishment order at a replenishment station
			\item \textbf{Pick Pod Selection (PPS)}: selection of pods to use for picking the pick orders assigned at a picking station
		\end{itemize}
		\item \textbf{Pod Repositioning (PR)}: assignment of an available storage location to a pod that needs to be brought back to the inventory area
	\end{itemize}
	\item \textbf{Task Allocation (TA)}: assignment of tasks from \textit{Task Creation} and additional tasks such as idling to robots
	\item \textbf{Path Planning (PP)}: planning of the paths for the robots to execute.
\end{itemize}

We simulate operational problems in our simulation framework RAWSim-O, except ROA and RPS in the replenishment process, since we concentrate in this paper on the picking process. And it is useful to see the effects that the methods have on the picking process without the replenishment process. Furthermore, the decision rules we apply for each decision problem are described in Section~\ref{subsec:rules}.
\subsection{Objective of POA and PPS} \label{subsec:obj}

As mentioned in Section~\ref{sec:intro}, we consider an efficient order picking system -- a system using minimal time to handle more orders -- as suggested in \cite{van2018designing} for picker-to-parts systems. We understand the efficient order picking system in an RMFS is a system to keep pickers busy by providing pods to pick from. In order to achieve that, we want to reduce the idle time between changes of pods; therefore, we aim at minimizing the number of visits by pods to stations (in short: \textit{pod-station visits}) for given sets of orders for the POA and PPS problems in this paper. We expect that a reduction in pod-station visits will increase pile-on (the number of picks per handled pod) while decreasing the distances driven by robots. So the operating costs of an RMFS can be reduced as well, such as electricity and wear and tear of the robots. Moreover, according to the analysis of \cite{Boysen.2017} the number of pod-station visits seems to be a good indicator of the necessary size of the robot fleet. In other words, the fewer pod-station visits we have, the fewer robots we need. So the setup costs of an RMFS can be reduced as well. The details of performance analysis can be found in Section~\ref{sec:results}.

\subsection{An example for sequential and integrated POA and PPS} \label{subsec:example}

In the following we describe Example~\ref{ex:serquencial-vs-integrated}, and explain how to get the minimal pod-station visits solutions for this example from sequential POA and PPS and integrated POA and PPS.

\begin{example}\label{ex:serquencial-vs-integrated}
Figure~\ref{fig:t_0} illustrates a small problem to fulfill four orders 1, 2, 3 and 4. The different colors represent different SKUs. For simplicity, the quantity of each SKU in the orders is one. We have in total two picking stations. There is one empty tote at each station. In this example we assume that each tote can hold three items. Pod 1 is currently at station 1 and pod 2 at station 2, while pods 3 and 4 are in the storage area.
\end{example}
\paragraph{Sequential POA and PPS}
In the sequential POA and PPS, we use the same decision rule, \textit{Pod-Match}, as in \cite{RawSimDecisionRules}, which assigns the orders from the backlog to a station so that the items for the orders best match the pods that are already assigned to that station. Note that there is another more common decision rule in \cite{RawSimDecisionRules} (called \textit{Common-Lines}) and \cite{Wurman.2008} grouping similar orders at picking stations in POA. However, the decision rule Pod-Match for POA is shown to perform better in \cite{RawSimDecisionRules}, since this rule uses information about assigned pods at stations in addition to information about orders in the backlog.
 
In Example~\ref{ex:serquencial-vs-integrated}, in the POA problem we assign orders 2 and 3 to station 1 (Figure~\ref{fig:t_1}), since two of their items can be picked from pod 1 -- the pod that is already at station 1. For the same reason, we assign orders 1 and 4 to station 2.
To fulfill the assigned orders, both pods from the storage area, pods 3 and 4, are needed at each station. After items from pods 1 and 2 are picked, they return to the storage area. In the PPS, pod 3 visits station 1, while pod 4 visits station 2 (Figure~\ref{fig:t_2}). After picking in both stations, pods 3 and 4 switch stations so that the last item of each order can be picked (Figure~\ref{fig:t_4}). In total, 6 pod-station visits were necessary to fulfill both orders in this example, therefore the pile-on can be calculated as 12 picks/6 pods = 2 picks/pod.
\begin{figure}[h]
	\begin{subfigure}[t]{\textwidth}
		\hspace{+0.3cm}
		\includegraphics[width = 0.3\textwidth, left]{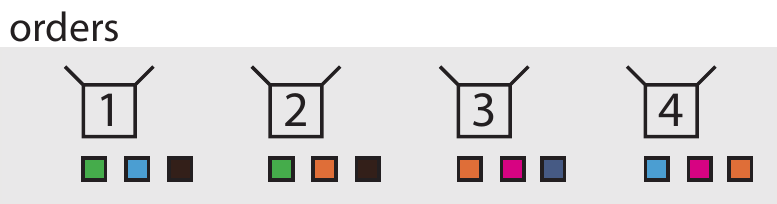}
	\end{subfigure}
	\newline
	\centering	
	\begin{subfigure}[b]{0.4\textwidth}
		\includegraphics[width = \textwidth]{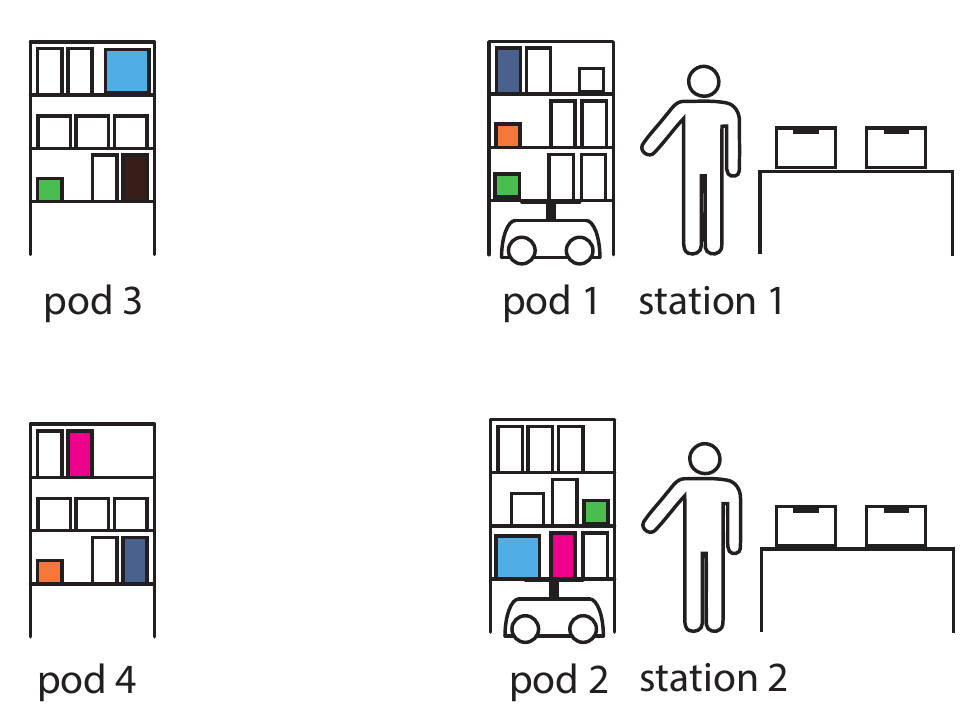}
		\caption{Initial state \newline}	
		\label{fig:t_0}
	\end{subfigure}
	\rulesep
	\begin{subfigure}[b]{0.4\textwidth}
		\includegraphics[width = \textwidth]{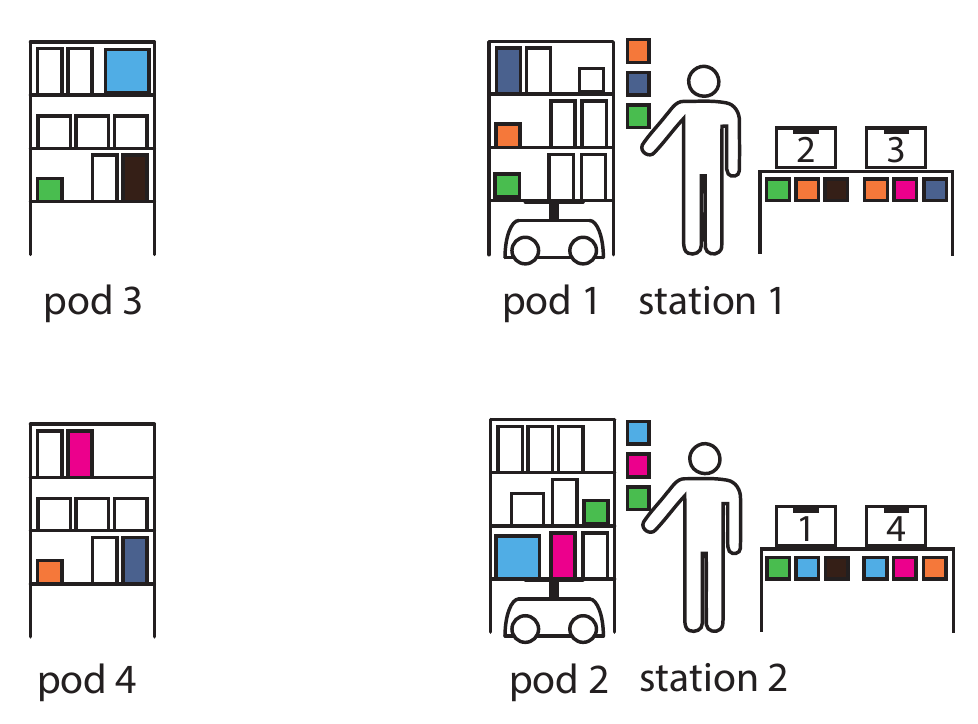}
		\caption{POA: orders 2 and 3 $\rightarrow$ station 1, orders 1 and 4 $\rightarrow$ station 2}	
		\label{fig:t_1}
	\end{subfigure}
	\begin{subfigure}[b]{0.4\textwidth}
		\includegraphics[width = \textwidth]{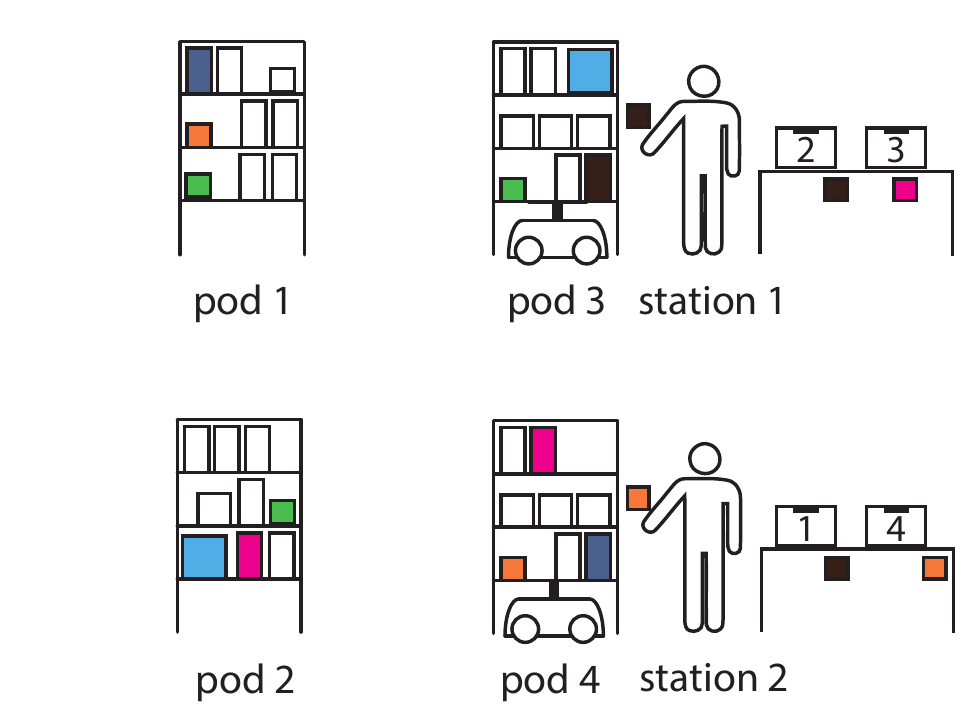}
		\caption{PPS I: pod 3 $\rightarrow$ station 1, pod 4 $\rightarrow$ station 2}	
		\label{fig:t_2}
	\end{subfigure}
	\rulesep
	\begin{subfigure}[b]{0.4\textwidth}
		\includegraphics[width = \textwidth]{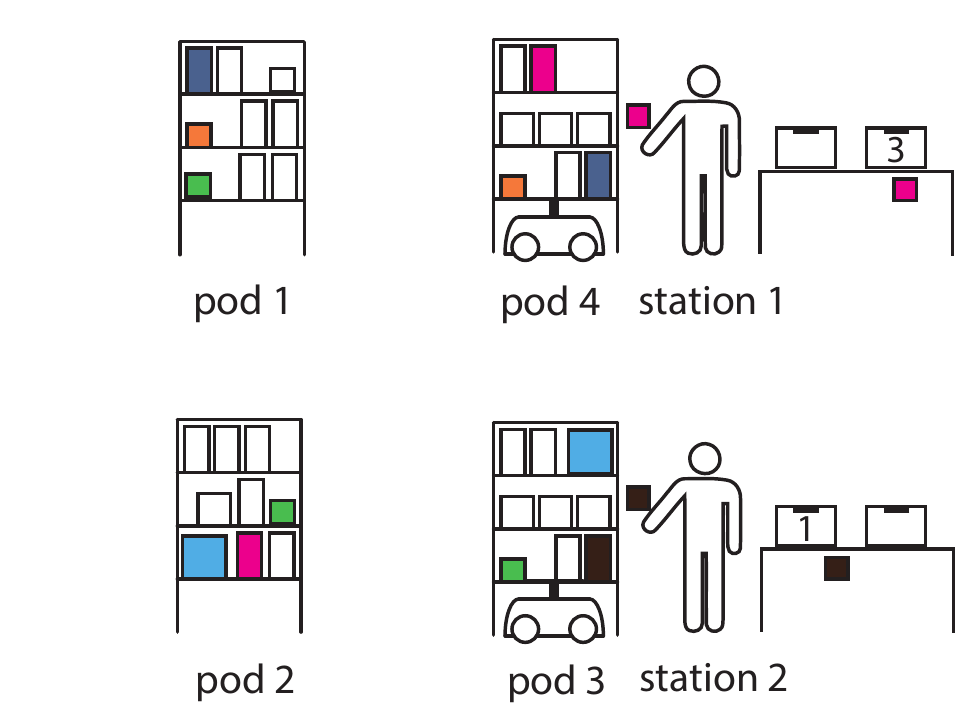}
		\caption{PPS II: pod 4 $\rightarrow$ station 1, pod 3 $\rightarrow$ station 2}	
		\label{fig:t_4}
	\end{subfigure}
	\caption{An example for the sequential POA and PPS.}
	\label{fig:example_sequential}
\end{figure}
\paragraph{Integrated PPS and POA}
In the integrated PPS and POA approach, we have more information while assigning orders to stations, since pods and orders are assigned to stations at the same time. This allows us to find optimal solutions that might not be intuitive at first glance and would not be found by the sequential POA and PPS. Note that we use information about all pods, including assigned ones at stations and unassigned ones in the storage area.

Using the same initial state as in the previous explanation of the sequential POA and PPS in Example~\ref{ex:serquencial-vs-integrated} (see Figure \ref{fig:t_0_integrated}), we integrate these two decisions and assign orders and pods such that the number of pod-station visits is minimized. This leads to the assignment of orders 1 and 2 and pod 3 to station 1 and orders 3 and 4 and pod 4 to station 2 (see Figures~\ref{fig:t_1_integrated} and \ref{fig:t_2_integrated}). This results in a pile-on of 3 (12 picks/4 pods) compared to 2 (12 picks/6 pods) in the sequential example and only requires 4 pod-station visits to fulfill all orders instead of 6. 

\begin{figure}[h]
	\begin{subfigure}{0.3\textwidth}
		\hspace{-2.1cm}
		\includegraphics[width = \textwidth, left]{orders_new.pdf}
	\end{subfigure}
	\newline
	\centering
	\begin{subfigure}[b]{0.4\textwidth}
		\includegraphics[width = \textwidth]{t_0_new.pdf}
		\caption{Initial state \newline \newline \newline}	
		\label{fig:t_0_integrated}
	\end{subfigure}
	\rulesep
	\begin{subfigure}[b]{0.4\textwidth}
		\includegraphics[width = \textwidth]{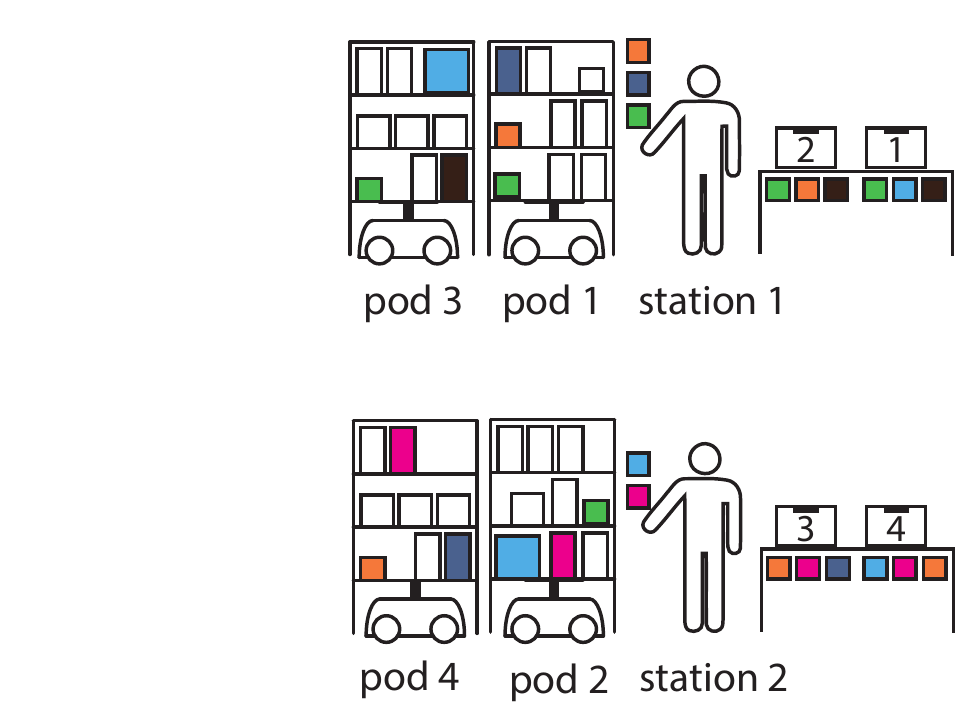}
		\caption{Integrated I: orders 1,2 and pod 3 $\rightarrow$ station 1, orders 3,4 and pod 4 $\rightarrow$ station 2; picking items from the existing pods 1 and 2 visited at stations}	
		\label{fig:t_1_integrated}
	\end{subfigure}
\rulesep
\begin{subfigure}[b]{0.4\textwidth}
	\includegraphics[width = \textwidth]{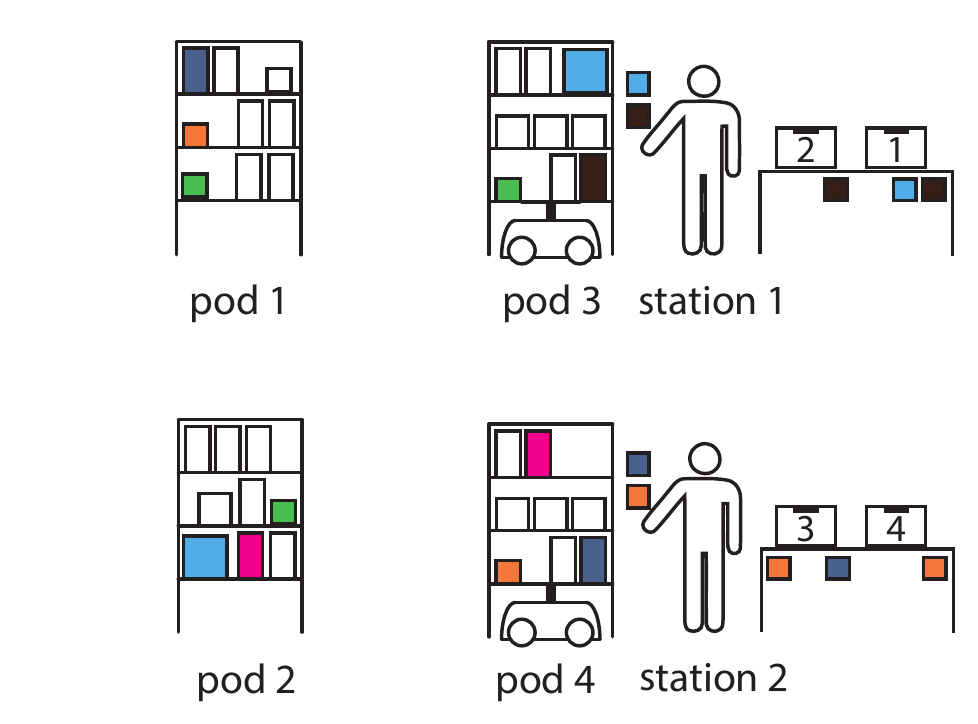}
	\caption{Integrated II: picking items from pod 3 at station 1, while picking items from pod 4 at station 2}	
	\label{fig:t_2_integrated}
\end{subfigure}
	\centering
	\caption{Same example as in Figure~\ref{fig:example_sequential} (see Figure \ref{fig:t_0_integrated}), but the decision is made by the integrated POA and PPS (see Figures \ref{fig:t_1_integrated} and \ref{fig:t_2_integrated}).}
	\label{fig:example_integrated}
\end{figure}

Based on this example, we can see the benefit of integrating POA and PPS by using information about the inventory of all pods in these decisions. Therefore, we present a mathematical model in this paper that integrates POA and PPS for multiple stations and takes information about the inventory of all pods into account.



\section{Mathematical model} \label{sec:model}
In this section, we describe the assumptions in Section~\ref{subsec:assumption} before we present our mathematical model of integrated POA and PPS (we call it the \textit{integrated model}), and extend it with two variants of allowing split orders.

\subsection{Assumption} \label{subsec:assumption}
\paragraph{SKUs} All different SKUs in orders are available in pods. We assume that the quantity of the order line for each SKU is one. This assumption is consistent with common practice, since the number of items per order line is low. If a pod contains a SKU, then we assume that there are enough items in that pod to fulfill all orders for that SKU.

\paragraph{Split order} Splitting an order means separating the original order into two or more parts (up to the number of SKUs in the order). If an order is not split, we ensure that all order lines in that order are assigned for picking at the same station (within a time period). If an order is split, this constraint is relaxed by allowing order lines for that order to be assigned to more than one picking station or more than one time period. There are two variants of a split order:
\begin{description}
	\item[split among stations:] all order lines for a pick order are assigned in the same period but may be assigned to different picking stations (see Example 1 in Section \ref{sec:intro})
	\item[split over timesteps:] order lines for a pick order may be assigned in different time periods and to different picking stations (see an example in \ref{app:example_timesplit}) 
\end{description}

\paragraph{Capacity of a picking station} Commonly, the capacity of a picking station is defined as the number of orders that can be handled at a time (\textit{order capacity}). According to \cite{Wulfraat.2012}, the typical station can support 6 to 12 orders to be picked at a time. The introductory example of split orders shows that traditional order capacity is incompatible with split orders, since simply counting the number of assigned orders does not work anymore when only parts of an order are assigned to the station. Instead, we introduce in this paper a new way to define the capacity of a picking station -- limited by the number of items to be handled at a time -- that works for both, whole orders and split orders. We call this type of capacity \textit{item capacity}. Another advantage of item capacity is a fairer distribution of workload among all stations, since the number of assigning items equals the number of picks. Note that the number of items in each order differs and reflects the number of picks.

\paragraph{Capacity of a packing station} Orders that are not split can be packed directly by packers as soon as they arrive at the packing stations. Split orders require storage space on shelves at packing stations to wait until all parts of the order are picked. Once all parts of a split order are picked, it can be packed and one space on the shelf becomes free for the next split order. The capacity of a packing station is therefore defined as the number of shelves multiplied by the number of boxes which can be stored on a shelf. We set the total capacity of all packing stations to a parameter $C$, and we assume it is large enough for all necessary split orders in this paper. This assumption is supported by the calculation in \ref{app:cal_cap_packing}. In our calculation, up to 78 split orders can be stored on a shelf. And usually, in practice, there is more than one packing station. If more split orders are required, then additional shelves can be installed at packing stations. However, the situation might differ from one company to another. Therefore, our model can be easily extended to support a limited packing capacity, as shown in Section~\ref{subsec:practical}.

\paragraph{Conveyor} We assume that the conveyors between picking and packing stations are long enough to temporarily store orders and parts of them. The conveyors serve as a buffer to synchronize the picking and packing stations.
\paragraph{Maximal order size}We assume that every order in the backlog can fit into some picking station. That means the maximal item capacity of the largest picking station is not smaller than the number of items of the largest order. 
\paragraph{Queue} There is a queue at each station. The space in a queue is limited. Each time a pod leaves the queue, one pod can be added at the end of the queue. Pods leave a station once their inventory cannot be used anymore to fulfill any further assigned orders. 
\paragraph{Period} Once there is enough free item capacity at a station, the time period is changed from $t$ to $t+1$ for all $t \geq 1$. The required amount of free item capacity is defined as the capacity that is needed to fit the smallest available order. In $t=0$, no orders are assigned or picked at any picking station. All pods are in the storage area, so there are no pods waiting at picking stations. At $t=1$ we start to assign orders from the backlog and pods to picking stations. Each time $t$ changes, the current situation in the warehouse (such as which pods are currently in storage or on their way to stations, free capacity at stations, inventory of pods, order backlog) is updated and used to compute the next decisions. This way, we can handle errors or delays in the execution of previous decisions. The model described in Section \ref{sec:model} is solved in each period $t$ using information about the current state of the warehouse.

\paragraph{Shared storage policy} Items of the same SKU are randomly spread over multiple pods. 
In \cite{Boysen.2017}, where this policy is called \textit{mixed-shelves storage}, the authors showed that this policy is efficient.

\paragraph{Pods selection} Our model computes the smallest possible set of pods to fulfill all assigned orders at each station in each period, without considering the distance between the selected pods and picking stations. 

\paragraph{Sequencing of pods} During each period $t$ we need to know the sequence of pods at each station. As our model only calculates the optimal sets of pods for each station, we use the following policy to create a sequence of pods: Robots begin immediately to carry all assigned pods to the respective
stations and the sequencing of pods is decided by the order of their arrival at the station. This ensures that station idle times are kept at a minimum. 



\subsection{Integrated model} \label{subsec:integtrated}
Firstly, we define the notation for the following model.
\newcommand{\SetOfPodsByStation}[1]{\mathcal{P}_{#1}}
\newcommand{\SetOfPodsBySku}[1]{\mathcal{P}_{#1}^\text{SKU}}

\textbf{Sets}:\\
\begin{tabularx}{\textwidth}{>{\hsize=0.45\hsize}X >{\hsize=1.55\hsize}X}
	$\mathcal{P}$ & Set of pods \\
	$\mathcal{S}$ & Set of currently available picking stations \\
	$\mathcal{P}_s$ & Set of pods $\mathcal{P}_s \subseteq \mathcal{P}$ that are currently at station $s$\\
	$\SetOfPodsBySku{i}$ & Set of pods 	$\SetOfPodsBySku{i} \subseteq \mathcal{P}$ that include SKU $i$\\
	$\mathcal{O}$ & Set of current orders in the backlog \\
	$\mathcal{I}_o$ & Set of SKUs $\mathcal{I}_o \subseteq \mathcal{I}$ that constitutes an order $o \in \mathcal{O}$ \\
\end{tabularx}

\textbf{Parameters}:\\
\begin{tabularx}{\textwidth}{>{\hsize=0.45\hsize}X >{\hsize=1.55\hsize}X}
	$C_s$ & Current capacity of each picking station $s \in \mathcal{S}$
\end{tabularx}

\textbf{Decision variables}:\\
\begin{tabularx}{\textwidth}{>{\hsize=0.2\hsize}X >{\hsize=1.8\hsize}X}
	$x_{ps}$ & 
	$
	\left\{
	\begin{array}{cl}
	1 , & \mbox{pod $p \in \mathcal{P}$ is assigned to station $s \in \mathcal{S}$}\\
	0 , & \mbox{else}
	\end{array}
	\right.
	$
	\\
	$y_{os}$ & 
	$
	\left\{
	\begin{array}{cl}
	1 , & \mbox{order $o \in \mathcal{O}$ is assigned to station $s \in \mathcal{S}$}\\
	0 , & \mbox{else}
	\end{array}
	\right.
	$
	\\
	$y_{ios}$ & 
	$
	\left\{
	\begin{array}{cl}
	1 , & \mbox{SKU $i \in \mathcal{I}_o$ of order $o \in \mathcal{O}$ is assigned to station $s \in \mathcal{S}$}\\
	0 , & \mbox{else}
	\end{array}
	\right.
	$
	\\
	$u_s$ & Amount of unused capacity for a station $s \in \mathcal{S}$  \\
\end{tabularx}
\\

The integrated model is invoked in the simulation each time the time period $t$ is changed. However, for simplicity the parameter $t$ is not used in the model. Note that all sets, parameters and decision variables may change for each time period $t$.

\begin{align}
\text{Min} \quad  & \sum_{p \in \mathcal{P}}\sum_{s \in \mathcal{S}} x_{ps} + \sum_{s \in \mathcal{S}}W_u\cdot u_s 
\label{eq:oap-1} \\
\text{s.t.} \quad 
& y_{os} = y_{ios}, \; \forall i \in \mathcal{I}_o, o \in \mathcal{O}, s \in \mathcal{S} \label{eq:oap-3}\\
& \sum_{s \in \mathcal{S}} y_{os} \leq 1, \; \forall o \in \mathcal{O} \label{eq:oap-4}\\
& 
\sum\limits_{o \in \mathcal{O}} \sum\limits_{i \in \mathcal{I}_o} y_{ios} + u_s= C_s, \; \forall s \in \mathcal{S} \label{eq:oap-5} \\
& \sum_{p \in \SetOfPodsBySku{i}} x_{ps} \geq y_{ios}, \; \forall i \in \mathcal{I}_o, o \in \mathcal{O}, s\in \mathcal{S} \label{eq:oap-9} \\
& x_{ps} = 1, \; \forall p \in \mathcal{P}_s, s \in \mathcal{S} \label{eq:oap-10} \\
& y_{os} \in \lbrace 0, 1 \rbrace, \; \forall o \in \mathcal{O}, s \in \mathcal{S} \label{eq:oap-12}\\
& y_{ios} \in \lbrace 0, 1 \rbrace, \; \forall i \in \mathcal{I}_o, o \in \mathcal{O}, s \in \mathcal{S} \label{eq:oap-13}\\
& x_{ps} \in \lbrace 0, 1 \rbrace, \; \forall p \in \mathcal{P}, s \in \mathcal{S} \label{eq:oap-14} \\
& u_s \in \mathbb{Z} \geq 0, \; \forall s \in \mathcal{S} \label{eq:oap-16}
\end{align}

In the integrated model given above, we aim at minimizing the number of pods used in each period, while keeping the unused capacity of stations as low as possible.
Constraint set \eqref{eq:oap-3} sets the value of $y_{os}$ to 1 for all assigned order/station tuples and ensures that, if an order is assigned to a station, all order lines in the order are assigned to the same station as this order. And if an order is not assigned to any station, none of its order lines can be assigned to a station. Constraint set \eqref{eq:oap-4} ensures that each order can be assigned to at most one station. Constraint set \eqref{eq:oap-5} ensures that the number of assigned items equals the amount of available capacity minus the amount of unused capacity $u_s$ at each station $ s \in \mathcal{S}$. Each $u_s$ increases the value of the cost function (assuming $W_u > 0$). Constraint set \eqref{eq:oap-9} ensures that for each order line of an order that is assigned to a station, at least one pod $p \in \mathcal{P}_i$ is assigned to the same station. Constraint set \eqref{eq:oap-10} ensures that pods that were assigned in previous periods and are currently on their way to a station or in a station's queue stay assigned to that station. Constraint sets \eqref{eq:oap-12}--\eqref{eq:oap-14} ensure that the respective variables can only have binary values while \eqref{eq:oap-16} ensures that $u_s$ can only have non-negative integer values.
\begin{prop}\label{prop:existance-of-non-split}
	The integrated model always has a feasible solution.
\end{prop}
\begin{proof}
	See \ref{app:proofs}.
\end{proof}



\subsection{Split-among-stations model} \label{subsec:splitmodel}
This model is an extension of the integrated model from the previous subsection.  
Now we allow splitting the SKUs in an order between two or more stations. We need the following additional decision variables.\\
\textbf{Additional decision variables}:\\
\begin{tabularx}{\textwidth}{>{\hsize=0.2\hsize}X >{\hsize=1.8\hsize}X}
	$y_{o}$ & 
$
\left\{
\begin{array}{cl}
1 , & \mbox{order $o \in \mathcal{O}$ is assigned }\\
0 , & \mbox{else}
\end{array}
\right.
$
\\
$e_o$ & the number of additional assigned picking stations for an order $o \in \mathcal{O}$ \\
\end{tabularx}
\begin{align}
\text{Min} \quad  & \sum_{p \in \mathcal{P}}\sum_{s \in \mathcal{S}} x_{ps} + \sum_{s \in \mathcal{S}}W_u\cdot u_s 
\tag{\ref{eq:oap-1}}\label{eq:oap-21} \\
\text{s.t.} \quad 
&\eqref{eq:oap-5}-\eqref{eq:oap-16} \nonumber \\
& y_{os} \geq  y_{ios}, \; \forall i \in \mathcal{I}_o, o \in \mathcal{O}, s \in \mathcal{S} \tag{\ref{eq:oap-3}.1} \label{eq:oap-23}\\
& \sum_{s \in \mathcal{S}} y_{os} \geq {y}_o, \; \forall o \in \mathcal{O} \tag{\ref{eq:oap-4}.1}\label{eq:oap-24}\\
& \sum\limits_{s \in \mathcal{S}} y_{ios} = y_{o}, \; \forall i \in \mathcal{I}_o, o \in \mathcal{O} \label{eq:oap-26} \\
& y_{o} \geq  y_{os}, \; \forall o \in \mathcal{O}, s \in \mathcal{S} \label{eq:oap-2}\\
& \sum\limits_{i \in \mathcal{I}_o} y_{ios} \geq  y_{os}, \; \forall o \in \mathcal{O}, s \in \mathcal{S} \label{eq:oap-27} \\
& y_{o} \in \lbrace 0, 1 \rbrace, \; \forall o \in \mathcal{O} \label{eq:oap-11}
\end{align}

The cost function \eqref{eq:oap-1} and constraints \eqref{eq:oap-5}--\eqref{eq:oap-16} are carried over from the previous model. Constraint set \eqref{eq:oap-3} is relaxed, so that the order lines for an order don't have to be assigned to the same station anymore (see constraint set \ref{eq:oap-23}), but $y_{os}$ is still set to 1 for all stations the order is assigned to. Constraint set \eqref{eq:oap-4} is also relaxed, to allow for the assignment of an order to more than one station (see constraint set \ref{eq:oap-24}).
Constraint set \eqref{eq:oap-26} now ensures that if an order is active (at least one order line is assigned to a station), all of its order lines have to be assigned to stations. Constraint set \eqref{eq:oap-2} sets the value of $y_{o}$ to 1 for each order that is assigned to at least to a station. 
Constraint set \eqref{eq:oap-27} ensures that an order can only be assigned to a station if at least one order line of the order is assigned to that station. 
Constraint set \eqref{eq:oap-11} ensures that $y_o$ is a binary variable for each $o\in \mathcal{O}$.


\begin{prop}\label{prop:existance-of-station-split}
Every solution of the integrated model also solves the split-among-stations model.
There is always a feasible solution of the split-among-stations model.
The split-among-stations model always provides a solution that is better than, or equally good as, the integrated model.
\end{prop}
\begin{proof}
See \ref{app:proofs}.
\end{proof}


\subsection{Split-over-time model}
\label{subsec:timesplitmodel}
This model is an extension of the split-among-stations model from the previous subsection. Here we also allow every order to be split over different time periods. This means some SKUs for an order may be assigned in one period while the others will stay in the backlog to be assigned in later periods. We define one additional binary variable.\\
\textbf{Additional decision variable}:\\
\begin{tabularx}{\textwidth}{>{\hsize=0.2\hsize}X >{\hsize=1.8\hsize}X}
$y^b_{io}$ & 
$
\left\{
\begin{array}{cl}
1 , & \mbox{SKU $i \in \mathcal{I}_o, o \in \mathcal{O}$ is moved back to the backlog}\\
0 , & \mbox{else}
\end{array}
\right.
$
\\
\end{tabularx}
\begin{align}
\text{Min} \quad & \sum_{p \in \mathcal{P}}\sum_{s \in \mathcal{S}} x_{ps} + \sum_{s \in \mathcal{S}}W_u\cdot u_s 
\tag{\ref{eq:oap-1}}\label{eq:oap-31}\\
\text{s.t.} \quad 
& \eqref{eq:oap-23}-\eqref{eq:oap-24}, \eqref{eq:oap-5}-\eqref{eq:oap-16}, \eqref{eq:oap-2}-\eqref{eq:oap-215} \nonumber \\
& \sum\limits_{s \in \mathcal{S}} y_{ios} + y^b_{io} = y_{o}, \; \forall i \in \mathcal{I}_o, o \in \mathcal{O} \tag{\ref{eq:oap-26}.1} \label{eq:oap-36} \\
& y^b_{io} \in \lbrace 0, 1 \rbrace, \; \forall i \in \mathcal{I}_o, o \in \mathcal{O} \label{eq:oap-317} 
\end{align}

All constraints from the previous model are carried over, except for constraint set \eqref{eq:oap-26}.  
Constraint set \eqref{eq:oap-26} is relaxed to allow not assigning all order lines of an order at once (see constraint set \eqref{eq:oap-36}).
The new constraint set \eqref{eq:oap-317} ensures that the value of $y^b_{io}$ is binary.

\begin{prop}\label{prop:existance-of-station-time-split}
Every solution of the split-among-stations model also solves the split-over-time model.
There is always a feasible solution of the split-over-time model.
The split-over-time model always provides a solution that is better than, or equally good as, the split-among-stations model.
\end{prop}
\begin{proof}
See \ref{app:proofs}.
\end{proof}

\section{Computational evaluation} \label{sec:results}
In this section, we describe the parameters and results of the computational evaluation. We first describe the open-source simulation framework used for this paper in \Subsecref{framework}. Next, we show the results of the simulation in \Subsecref{results}. In \Subsecref{practical}, we make some remarks regarding the assumptions that were made in our computational evaluation from a practical point of view.
\subsection{Simulation framework} \label{subsec:framework}
In the following evaluation we use RAWSim-O from \cite{RawSim.1}, an open-source, agent-based and event-driven
simulation framework providing a detailed view of an RMFS. The source code is available at \url{www.rawsim-o.de}. Figure~\ref{fig:mu_simulationframework} shows an overview of the simulation process, which is managed by the core \textit{simulator} instance. The tasks of the simulator include:
\begin{itemize}
	\item updating \textit{agents}, which can resemble real entities, such as robots and stations
	\item passing decisions to \textit{optimizers}, which can either decide immediately or buffer multiple requests and release the decision later
	\item exposing information to a \textit{visualizer}, which allows optional visual feedback in 2D or 3D.
\end{itemize}

\begin{figure}[htb]
	\centering
	\includegraphics[width = 0.7\textwidth]{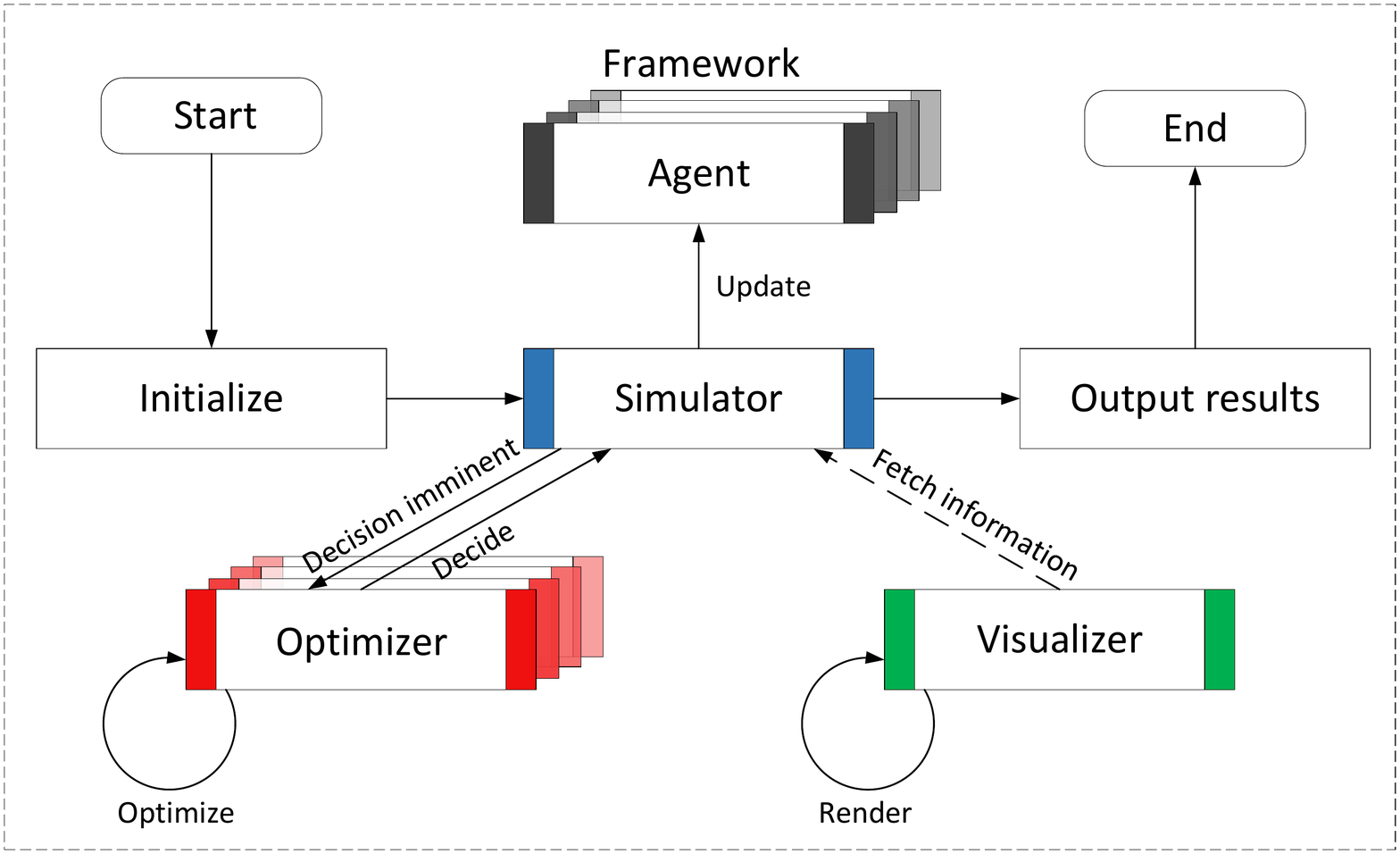}
	\caption{Overview of the simulation process (see \cite{RawSim.1}).}
	\label{fig:mu_simulationframework}
\end{figure}
\vspace{-0.5cm}
The hierarchy of decision problems regarding the assignment of replenishment orders, pick orders and pods to station is illustrated in  Figure~\ref{fig:problemdependencies}. If a replenishment order needs to be assigned to a replenishment station, the optimizers of ROA and RPS are responsible for choosing a replenishment station and a pod. This results in an insertion request, i.e. a request for a robot to bring the selected pod to the given workstation. If a picking order needs to be assigned to a picking station, a newly developed optimizer different to \cite{RawSim.1}, called \textit{POA \& PPS}, submits all necessary information to the model and converts its solution into extraction requests. Extraction requests contain both, the item that needs to be picked and the pod that it should be picked from. Note that in \cite{RawSim.1}, first the POA optimizer is called and the extraction requests are generated; after that, the PPS optimizer is called. In other words, pods are only assigned to stations after orders have already been assigned. Moreover, in the new optimizer, extraction requests for an order can be assigned to different stations to support split orders. The item capacity is the sum of the extract requests at one station for each (part of an) order that has at least one unfulfilled request at that station. Furthermore, the system generates a store request each time a pod needs to be transported back from a station to a storage location, and the PR optimizer decides on the storage location for that pod. The TA optimizer assigns robots to these tasks. All tasks result in robot trips, which are planned by a PP optimizer. 

\begin{figure}[h]
	\centering
	\includegraphics[width=0.85\textwidth]{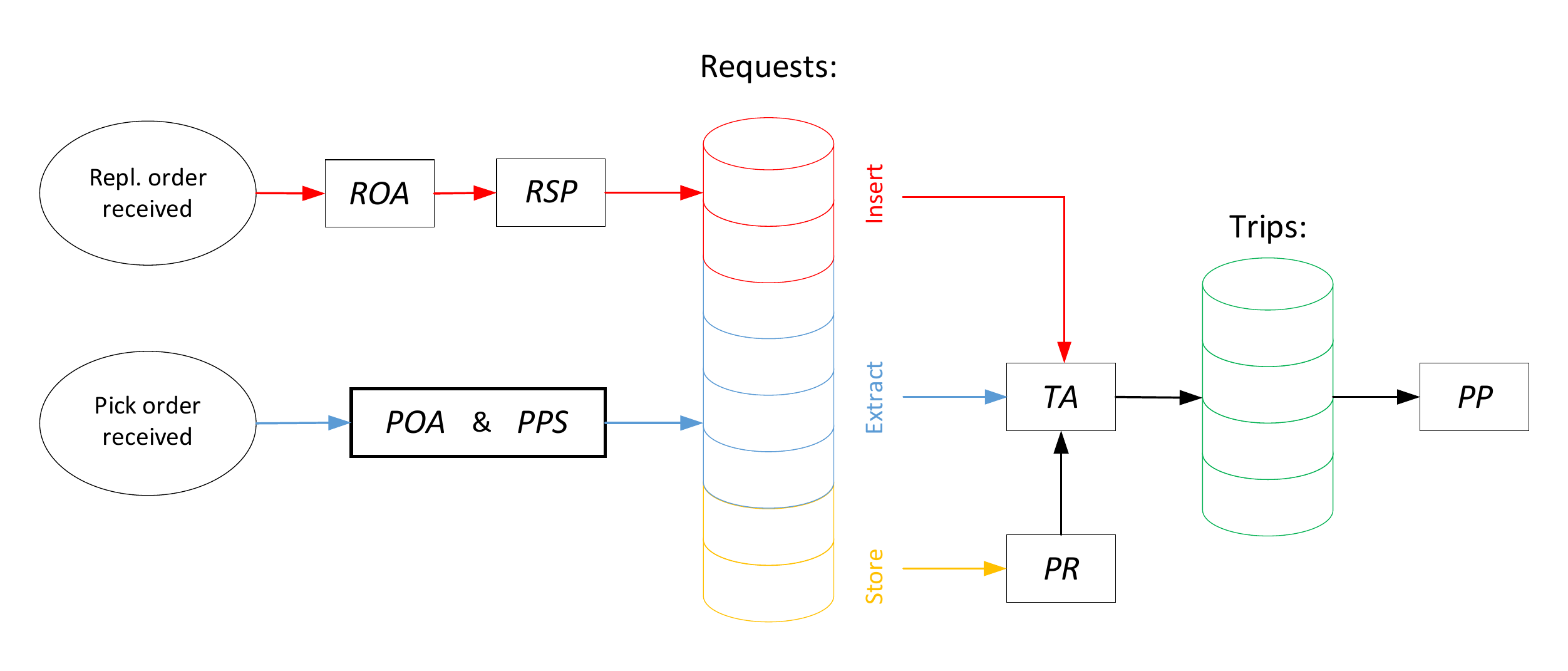}
	\caption{Order of decisions to be made after receiving a pick or replenishment order.}
	\label{fig:problemdependencies}
\end{figure}

The simulation framework conceptually consists of three different inputs: 
\begin{itemize}
	\item instance configuration: contains orders, initial inventory and available SKUs (see Section~\ref{subsec:intance})
	\item layout configuration: determines the characteristics and dimensions of the warehouse layout (see Section~\ref{subsec:layout})
	\item optimizer configuration: specifies the decision rules for all operational problems in an RMFS (see Section~\ref{subsec:rules}).
\end{itemize}
We describe these three inputs for our experiments in this paper in the following sections.
\subsubsection{Instance generation} \label{subsec:intance}
First, we describe how we generate SKUs and orders, and how to fill the pods. The set of SKUs is generated as $I$ = \{$i_1$, ... , $i_{|I|}$\}. For each order $o_1, ..., o_{|O|}$ the number of different SKUs in it is determined by a truncated (1 to $|I|$) geometric distribution with $p = 0.4$ (see Figure~\ref{fig:orderlen} for $|I|=20$). And the number of items for a SKU is set to 1. It is typical in online retailing that most of the orders contain very few line items, such as in \cite{bozer2018simulation} and \cite{onal2017modelling}. The SKUs in the order are then chosen by sampling without replacement from $I$ using the probability mass function of a geometric distribution with $p = 5/|I|$, to account for the varying demand for different SKUs. Figure~\ref{fig:skus} shows the probabilities of SKUs for $|I|=20$.

\begin{figure}[H]
	\centering
	\begin{minipage}{.5\textwidth}
		\centering
		\captionsetup{justification=centering}
		\includegraphics[width=1\linewidth]{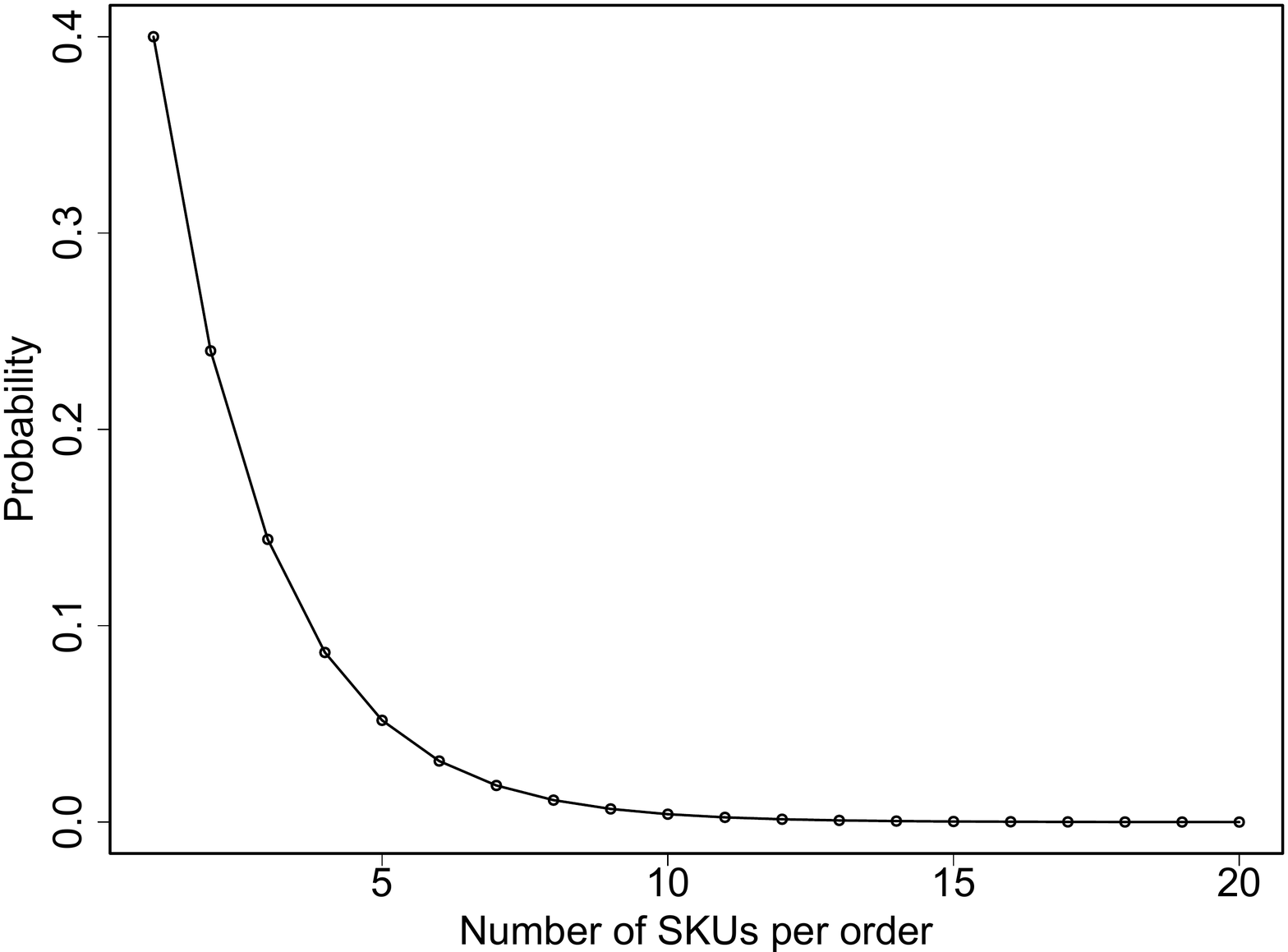}
		\captionof{figure}{Probabilities of order lengths for $|I|=20$.}
		\label{fig:orderlen}
	\end{minipage}%
	\begin{minipage}{.5\textwidth}
		\centering
		\captionsetup{justification=centering}
		\includegraphics[width=1\linewidth]{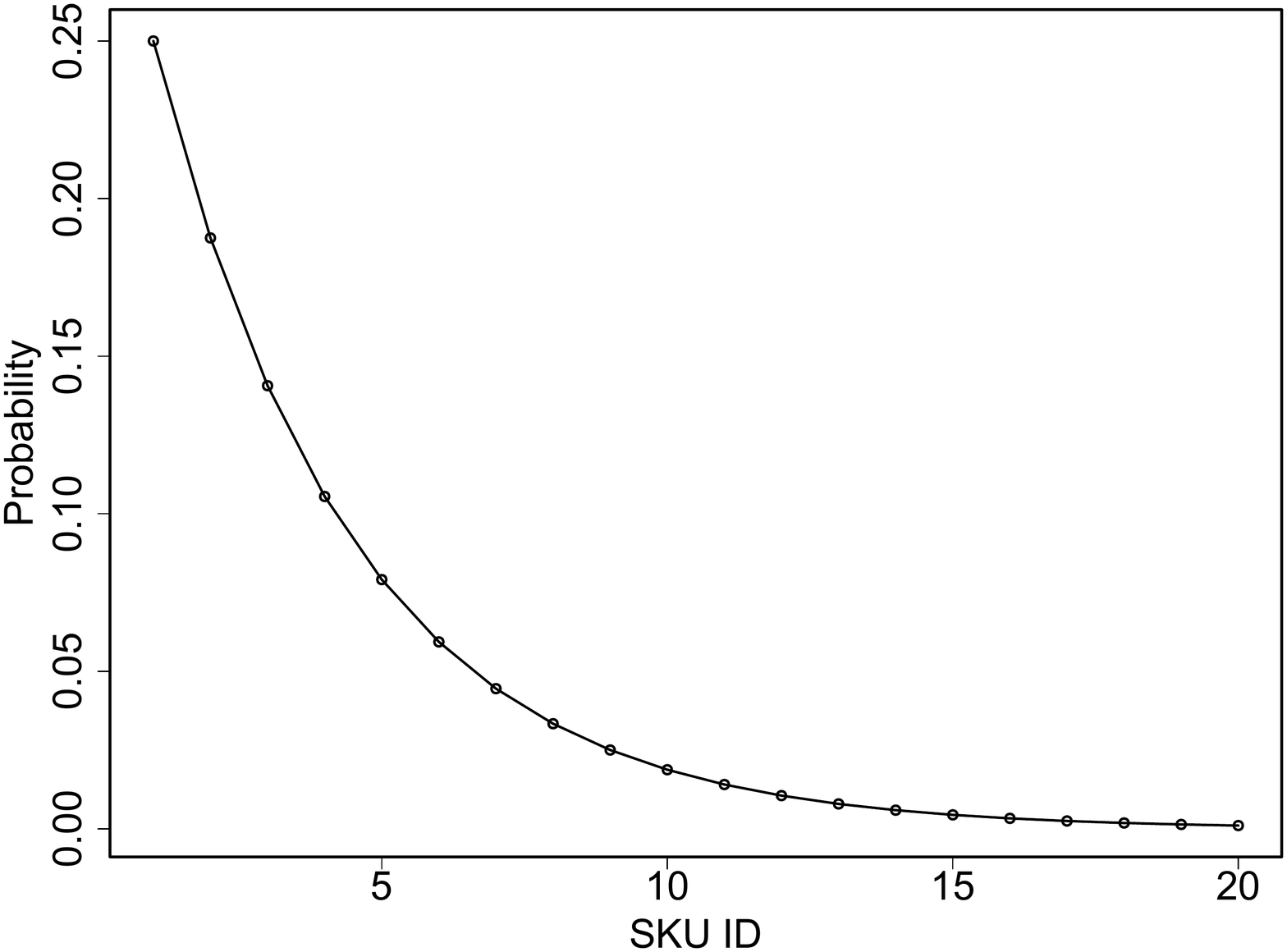}
		\captionof{figure}{Probabilities of SKUs for \\$|I|=20$.}
		\label{fig:skus}
	\end{minipage}
\end{figure}

A shared storage strategy (as described in \cite{bartholdi.2017}) is applied to fill the pods. To determine the initial inventory of the pods, lists of all SKUs in randomized order are concatenated until the combined list includes at least $|P|$ $\cdot$ $\alpha$ elements. Then, for each pod $p_1, ..., p_{|P|}$, the inventory is determined by cutting off the first $\alpha $ items in the list. 

We generate instances with different parameters from \Tabref{instance}. By considering all possible combinations of the parameters, we generate 24 instances. We test all methods in this paper with the pregenerated instances to see the efficiency of the different algorithms. In a real RMFS, the orders would not be known at the start of the simulation but would instead be received during runtime. Our methods are also compatible with this type of order-generation (see \Subsecref{practical}).

\begin{table}[H]
	\centering
	\begin{tabular}{ l l l }
		\hline
		Symbol & Description & Values\\
		\hline
		$|O|$ & Number of orders & 50, 150, 250 \\
		$|I|$ & Number of SKUs & 20, 100 \\
		$|P|$ & Number of pods & 50, 100 \\
		$\alpha$ & SKUs per pod & 2, 3 \\
		\hline
	\end{tabular}
	\captionof{table}{Instance parameters.}
	\label{tab:instance}
\end{table}


\subsubsection{Layout} \label{subsec:layout}
The layout is identical for all test instances and is illustrated in Figure~\ref{fig:Layout}: 428 pods (blue squares), 504 storage locations in $2\times4$ blocks, 4 picking stations (red squares) and 8 robots (gray circles). The length of station queues is 12. As we are focusing on the decisions made in the order picking process, we disregard the replenishment process. When the number of pods is higher than the value of $P$, only $P$ pods have an inventory and the empty pods are not taken into consideration. 
The capacity of picking stations is similar to that used in \cite{Boysen.2017}. \cite{Boysen.2017} used an order capacity of 6 orders. Since we use item capacity instead of order capacity, we multiply 6 by 2.5, the mean of the distribution used to determine the number of items per order (described in \Subsecref{intance}), which leads to an item capacity of 15. 
\vspace{0.5cm}
\begin{figure}[h]
	\centering
	\includegraphics[width=0.8\textwidth]{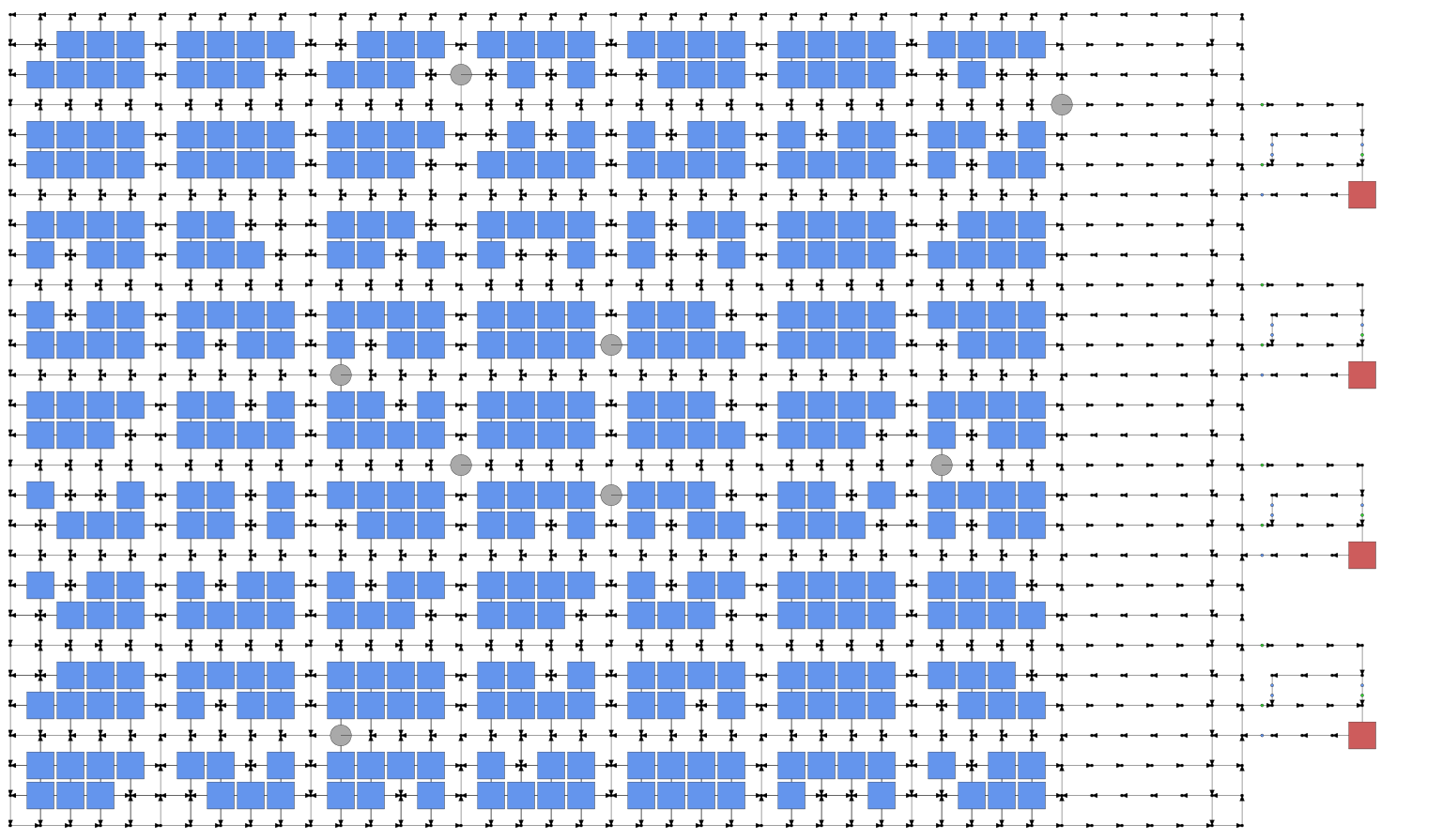}
	\caption{Simulation layout.}
	\label{fig:Layout}
\end{figure}
\subsubsection{Decision rules} \label{subsec:rules}
Table \ref{tab:decrules} lists the decision rules for all operational problems in the evaluation. 

\begin{table}[H]
	\centering
	\begin{tabular}{ l l l  }
		\hline
		Decision Problem & Sequential & Integrated (without/with split orders)\\
		\hline
		POA & Pod-Match & Model \\
		PPS &  Demand & Model \\
		ROA &  not relevant & not relevant \\
		RPS &  not relevant & not relevant \\
		PR &  Nearest & Nearest \\
		PP & WHCA$^*_n$  & WHCA$^*_n$ \\
		\hline
	\end{tabular}
	\captionof{table}{Decision rules}
	\label{tab:decrules}
\end{table}
The POA and PPS decision rules selected for the sequential approach are based on \cite{RawSimDecisionRules}, since these combinations achieved the best throughput. The rule \textit{Pod-Match} for the POA problem selects the pick order from the backlog that best
matches the pods already assigned to the station. Out of all pods that could fulfill at least part of an assigned order, the decision rule Demand for PPS chooses the pod whose inventory is in most demand when combining all unfulfilled orders.
The {\it ROA} and {\it RPS} problems are concerned with the replenishment process and thus not relevant to our tests.
The decision rule Nearest for {\it PR} assigns the closest parking space to each pod leaving a picking station.
For the {\it PP} problem, non-volatile WHCA* from \cite{pathplanning} is used. 
The model used for POA and PPS to test our integrated approach is the integrated model described in Section~\ref{sec:model} and both of its extensions for split orders.

\subsection{Simulation results} \label{subsec:results}
Each combination of method and instance is simulated ten times to lessen the effect of randomness caused by other optimizers and simulation components. Testing was done on 12 Intel Xeon X5650 Cores with 24 GB RAM. The following performance metrics are tested:
\begin{itemize}
	\item the number of pod-station visits per order
	\item the distance driven by robots per order
	\item pile-on: the number of picks per pod-station visit
	\item turn-over time: measured from the time an order is received to the time its last item is picked.
\end{itemize}

The parameter $W_u$ is set to 2, since we aim to fully utilize the picking station's capacity whenever possible. Any value larger than 1 would lead to the same outcome. Further parameters used in the simulation can be found in \ref{app:restpara}, such as parameters for robot movement.

Figure \ref{fig:podstationvisitsperorder} shows the average number of pod-station visits per order relative to the sequential approach for each instance set of 50, 150 and 250 orders. Compared to the sequential approach for the number of pod-station visits per order: The integrated model improves this performance by 20\% to 30\% for different instance sets, the split-among-stations model (in short in figures and tables: split) shows improvements of about 50\%, and the split-over-time model (in short in figures and tables: timesplit) improves on the sequential solution by 57\% to 80\% for different instance sets. 

Figure \ref{fig:distancesperorder} shows similar improvements for the average distances driven by robots to complete an order for each instance set. The correlation between pod-station visits and distances driven by robots confirms our assumption in Section~\ref{subsec:obj}, that distances driven by robots can be reduced by minimizing the number of pod-station visits. As the layout shown in \Figref{Layout} demonstrates, the distance between the inventory area and a picking station is in most cases greater than the distance between any two picking stations. Both a pod coming from the inventory area to a station and a pod coming from another picking station count as pod-station visits. In the sequential approach, pods may have to be transported directly between stations, because the orders that share the same pods are assigned to different stations. Therefore, there are pods moving between picking stations. In the integrated approach, we try to assign orders which share the same pods to the same stations; therefore, it reduces the number of pods moving between picking stations. This explains why the distances per robot in the integrated approach do not show as much of an improvement as pod-station visits per order compared to the sequential approach.

\vspace{-0.4cm}
\begin{figure}[H]
	\centering
	\begin{minipage}{.4\textwidth}
		\centering
		\captionsetup{justification=centering}
		\includegraphics[width = \textwidth]{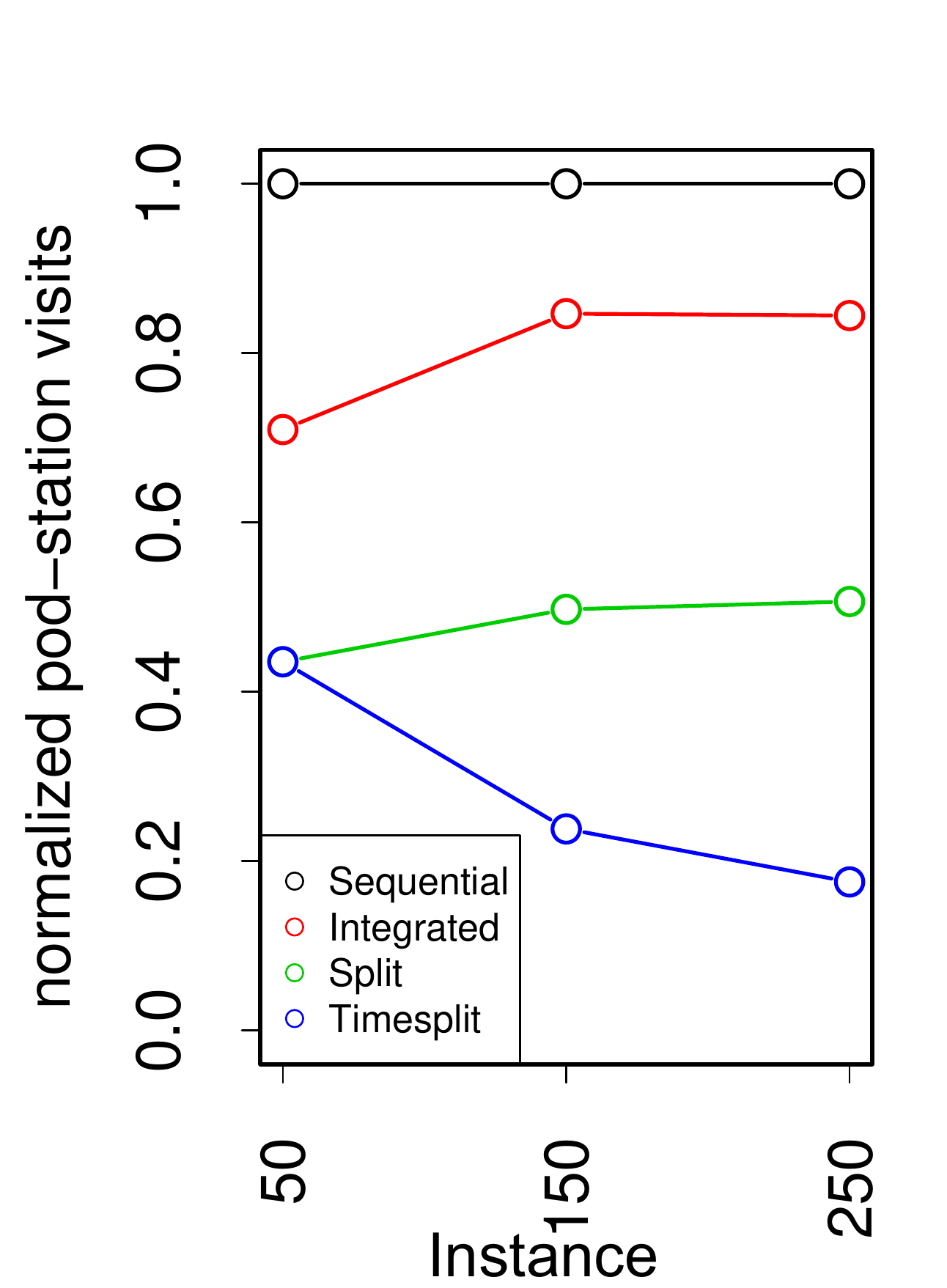}
		\captionof{figure}{Normalized pod-station visits per order.}
		\label{fig:podstationvisitsperorder}
	\end{minipage}%
	\begin{minipage}{.4\textwidth}
		\centering
		\captionsetup{justification=centering}
		\includegraphics[width = \textwidth]{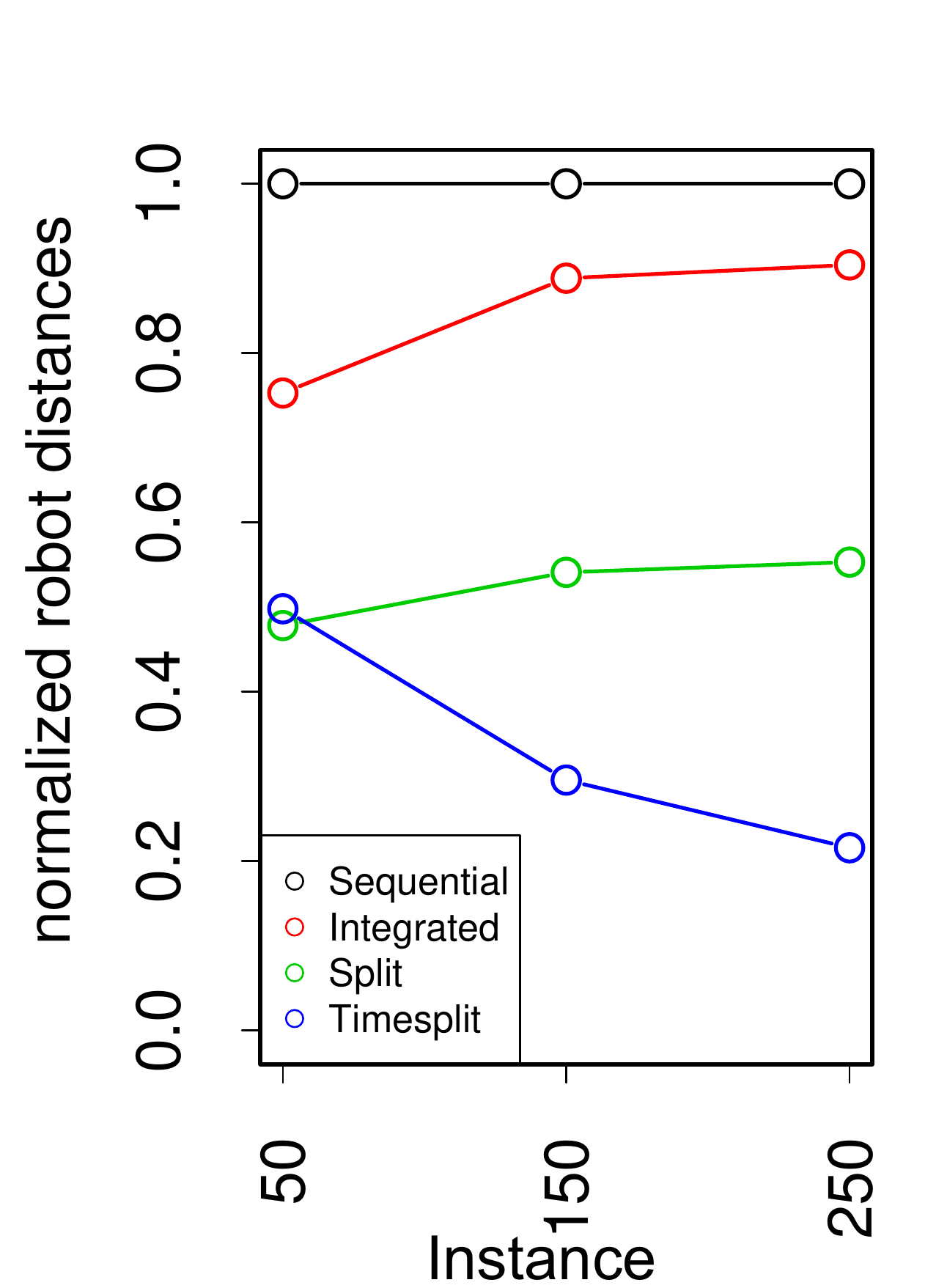}
		\captionof{figure}{Normalized robot distances per order.}
		\label{fig:distancesperorder}
	\end{minipage}%
\end{figure}
\vspace{-0.4cm}
Figure~\ref{fig:pileon} shows the average pile-on of all methods for each instance set of 50, 150 and 250 orders. Our approach with split orders causes more picks per pod-station visit (PSV), especially in the split-over-time model (up to 5.5 times as many picks per PSV). 

\vspace{-0.4cm}
\begin{figure}[H]
	\begin{center}
		\includegraphics[width =0.4\textwidth]{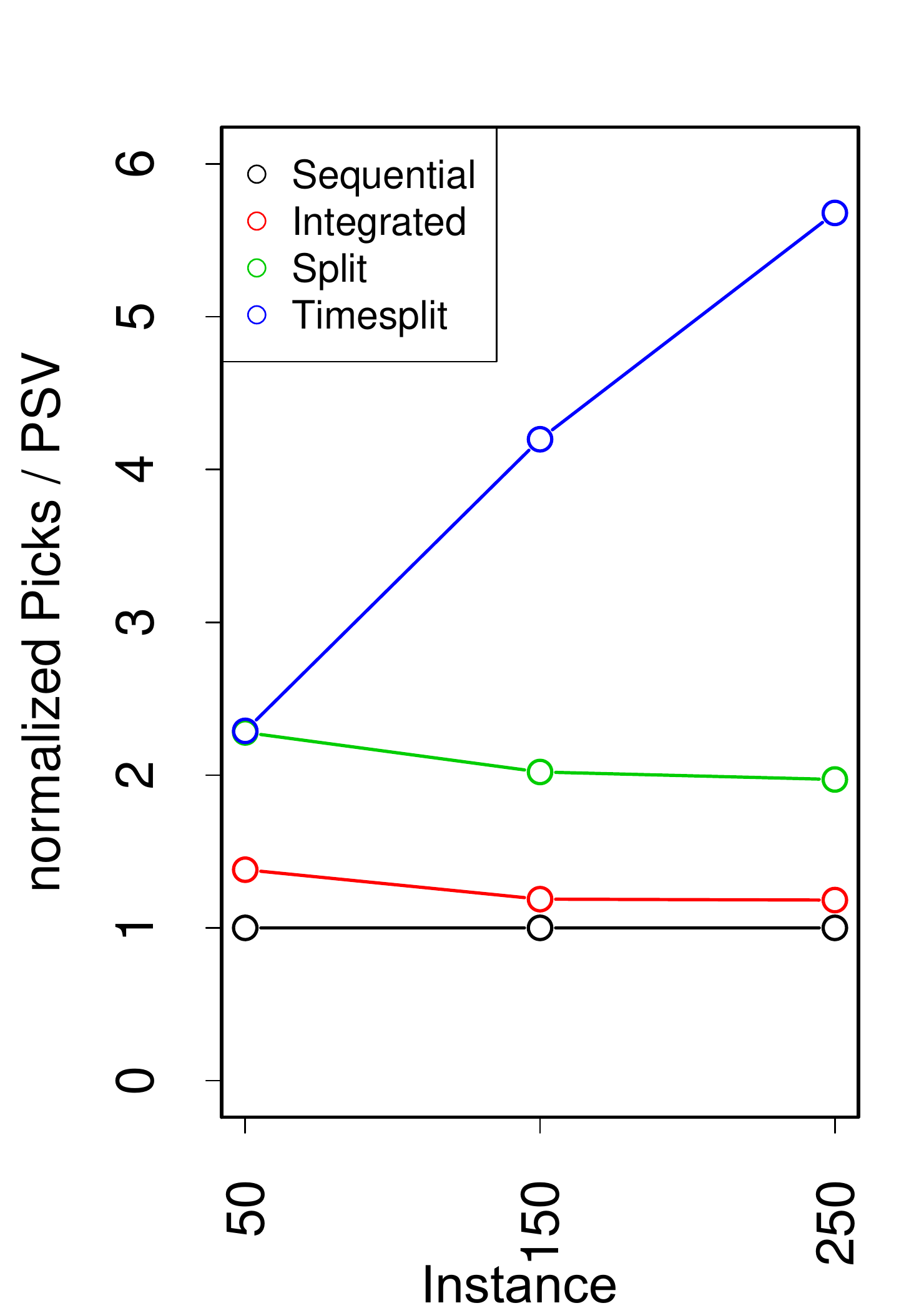}
		\caption{Normalized pile-on.}
		\label{fig:pileon}
	\end{center}
\end{figure}
\vspace{-0.4cm}

In Figure~\ref{fig:turnovertime} we separate the turn-over time into two parts, namely the time in the backlog (the lower part) and the time in the station for picking (the upper part). This metric cannot be directly applied to the split-over-time model, since parts of an order are assigned to stations at different time periods and some parts may spend more time in the backlog than others. Therefore, for the split-over-time results, the lower bar depicts the time until all parts of an order are assigned to a station while the upper bar shows how long it takes after that for the order to be fully picked. In the integrated and split-among-stations models, the orders spend less time in both the stations and the backlog, compared to the sequential approach. However, the turn-over time for orders in the split-over-time model is higher, since the parts of an order that can be picked with a low objective value are fulfilled first while the parts that need more pods to fulfill are pushed back. Therefore, it takes longer on average until an order is fully picked despite the picks/time being equal.

\begin{figure}[h]
	\centering
	\includegraphics[width=0.7\textwidth]{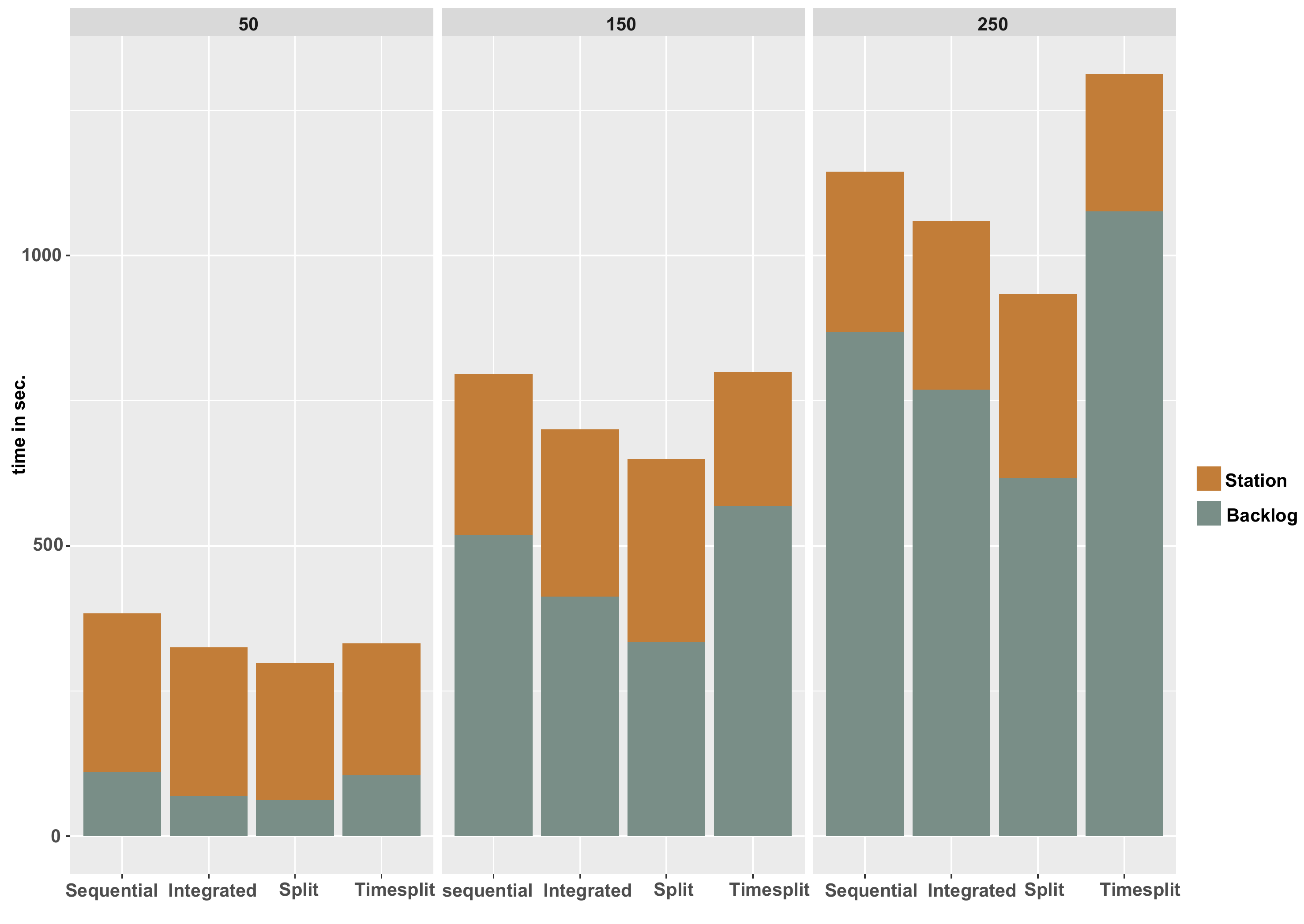}
	\caption{Turn-over time.}
	\label{fig:turnovertime}
\end{figure}

\Figref{comptime_total} illustrates the total computing times for each instance set of 50, 150 and 250 orders. The total computing time for the integrated approaches with and without split orders are much larger than those for the sequential approach. We divide the total computing time into the time at $t=1$ and the time of the remaining periods $t>1$. The period at $t=1$ takes a lot of time, since no orders or pods have previously been assigned to any stations and therefore more decisions are necessary than in the following periods. As this usually happens at most once a day (when the system is stopped and then restarted), we can consider it a warm-up. It is more interesting to see the periods $t>1$, since most decisions are made in those periods. Note that the decisions at $t=1$ in the split-over-time model are faster than those in the integrated and split-among-stations models. The reason is that the split-over-time model can easily find the best possible solution in the first timestep, where only one pod is needed at each station. Table~\ref{tab:reltosimtime} shows the total computing time of the decisions in $t>1$ in relation to the simulation time. We deem computing times acceptable as long as they are lower than the total simulation time. In our tests, computing the decisions took at most around 15\% of the simulation time (split-over-time, 250-order instance set). The average time for the assignment of one order is at most 1.5 seconds (see Table~\ref{tab:reltotimeperorder}). 
\begin{figure}[H]
	\begin{center}
		\includegraphics[width = 0.75\textwidth]{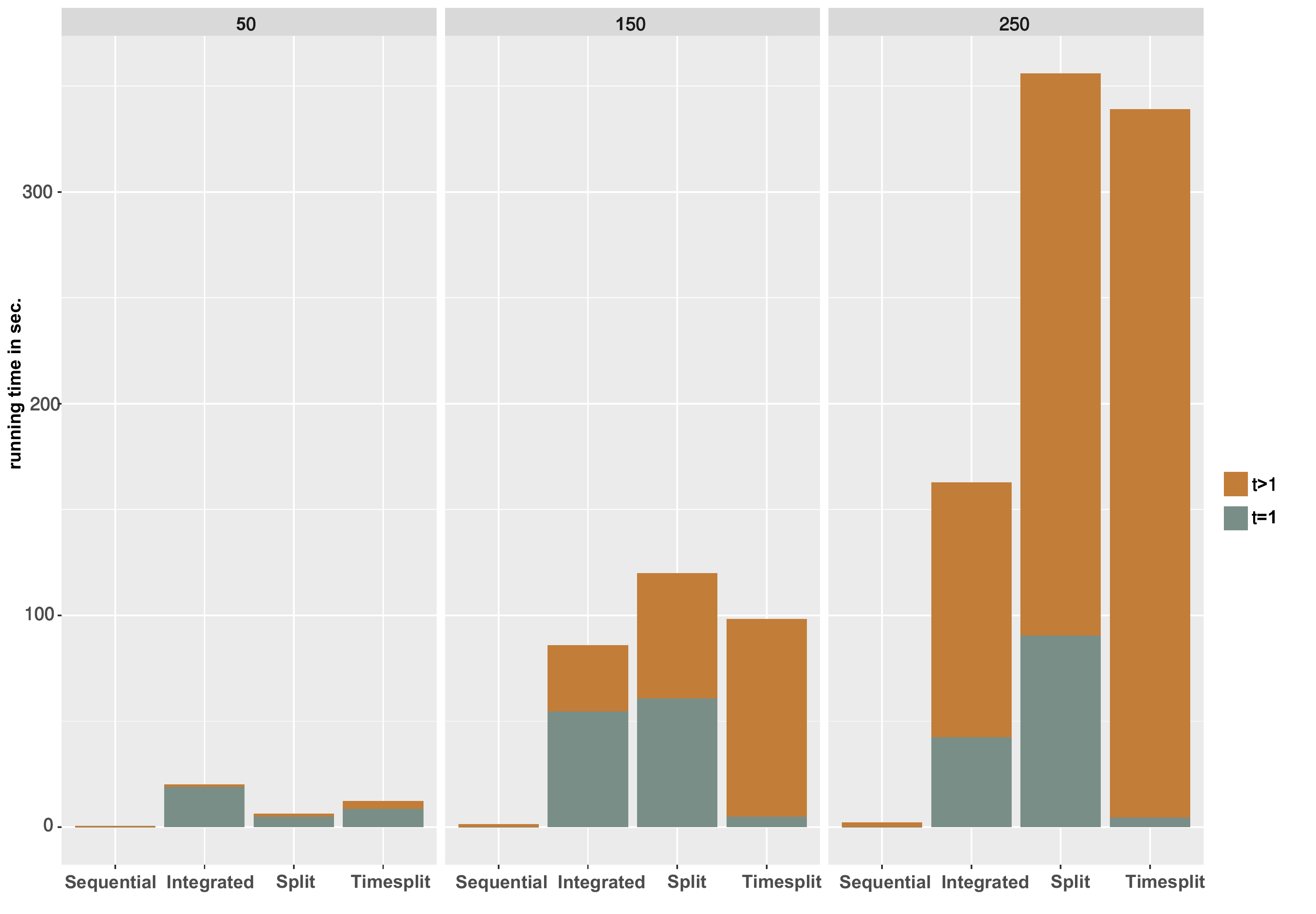}
		\caption{Computing time for the first period and others.}
		\label{fig:comptime_total}
	\end{center}
\end{figure}


\begin{table}[H]
	\centering
	\begin{tabular}{ l l l l l}
		\hline
		Instance set & Sequential & Integrated & Split & Timesplit\\
		\hline
		50 & 0.07\% & 0.17\% & 0.24\% & 0.61\% \\
		150 &  0.07\% & 1.4\% & 3.46\% & 6.41\% \\
		250 &  0.07\% & 3.39\% & 9.67\% & 14.67\% \\
		\hline
	\end{tabular}
	\captionof{table}{Computing time per simulation time for $t>1$. }
	\label{tab:reltosimtime}
\end{table}

\begin{table}[H]
	\centering
	\begin{tabular}{ l l l l l}
		\hline
		Instance set & Sequential & Integrated & Split & Timesplit\\
		\hline
		50 & 0.01 & 0.03 & 0.03 & 0.08 \\
		150 &  0.01 & 0.21 & 0.4 & 0.62 \\
		250 &  0.01 & 0.48 & 1.06 & 1.34 \\
		\hline
	\end{tabular}
	\captionof{table}{Computing time $t>1$ per order (in seconds).}
	\label{tab:reltotimeperorder}
\end{table}
\subsubsection{Acceleration method for larger instances} 
In order to reduce the computing time of our integrated models, especially for larger instances, we developed a prefiltering method which selects the $n \leq |\mathcal{O}|$ orders from the backlog that have the highest percentage of order lines for which a pod containing the SKU currently is at (or on its way to) a station. We test this method for different $n$ in all instance sets and model variants to analyze the effect $n$ has on computing time and the quality of the solution (number of pick-station visits). The results of different $n$ between 10 and 250 for 250-order instances in the split-among-stations model are illustrated in Figure~\ref{fig:prefiltering}. The results of other instance sets show a similar distribution for the integrated approaches. By using the prefiltering method, we can save up to 90\% in computing time while pod-station visits rise by 17\% compared to the results of the split-among-stations model without using the prefiltering method (see order count 10 in Figure~\ref{fig:prefiltering}). Even for $n$ = 10, the solution is still 40\% better than the sequential solution. Table~\ref{tab:rel} displays the computing time for t $>$ 1 relative to the simulation time.
\begin{figure}[h]
	\begin{center}
		\includegraphics[width=0.75\textwidth]{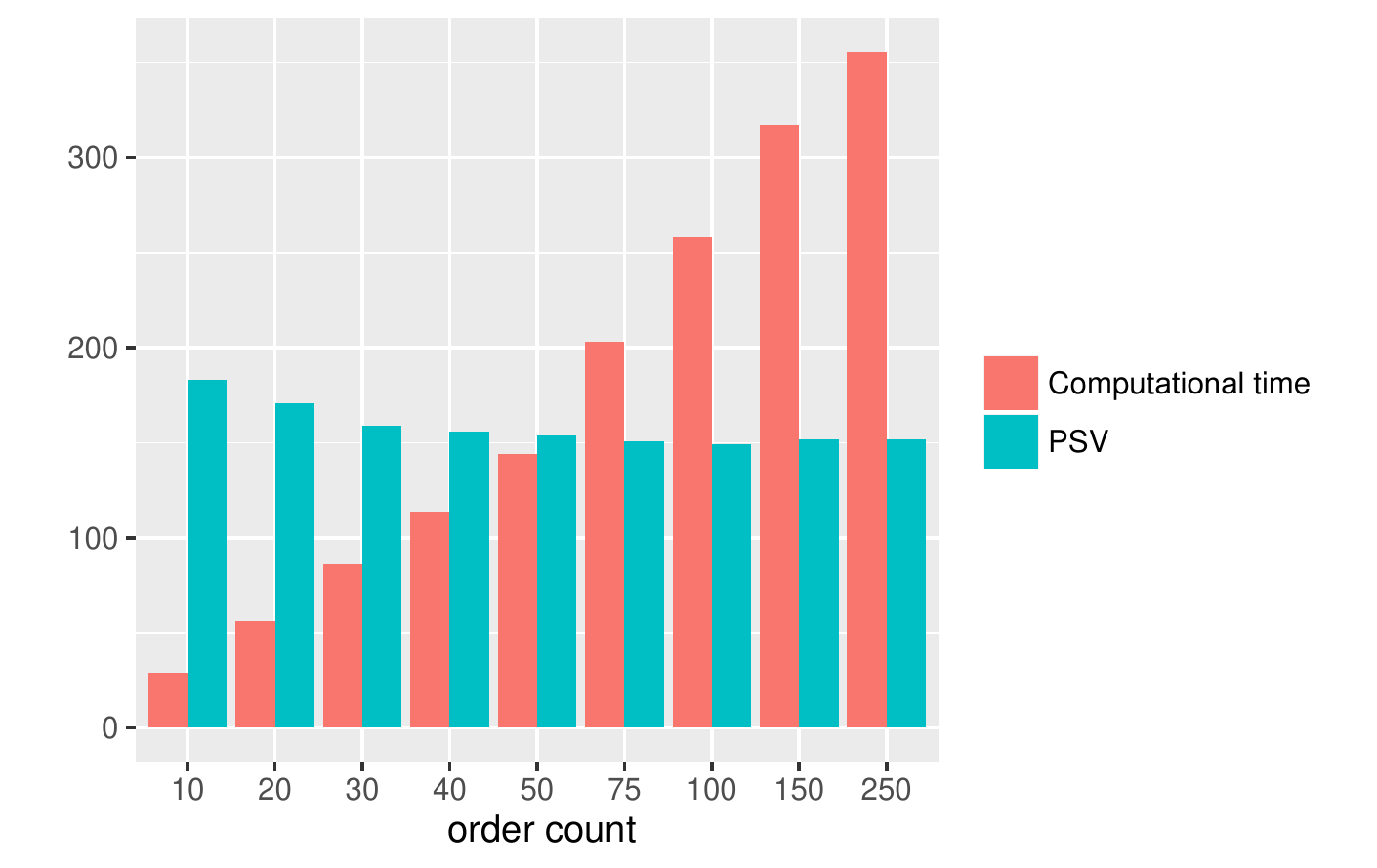}
		\caption{Computational time vs. the number of pick-station visits for the 250-order instance set by prefiltering $n$ orders ($n$ is from 10 to 250).}
		\label{fig:prefiltering}
	\end{center}	
\end{figure}
\begin{table}[H]
	\centering
	\begin{tabular}{l l l l l l l l l l}
		\hline
		$n$ &10 & 20 & 30 & 40 & 50 & 75 & 100 & 150 & 250 \\
		\hline
		&1.03\% & 2.0\% & 3.07\% & 3.92\% & 4.93\% & 6.97\% & 8.41\% & 9.64\% &9.67\%\\
		\hline
	\end{tabular}
	\captionof{table}{Relative running times compared to the simulation times for $t>1$. }
	\label{tab:rel}
\end{table}
Using the previously described prefiltering method, we test our integrated approaches for a larger instance set of 2000 orders. SKUs and orders were generated as described in Subsection~\ref{subsec:intance} and the layout shown in Subsection~\ref{subsec:layout} is used. We chose $n=10$ to reduce the running time. Figure~\ref{fig:psv_distance_2000_fil10_8bots} shows the relative distances driven by robots and combined pod-station visits compared to the sequential solution. Note that the prefiltering method is not applied to the sequential approach. This leads to the integrated solution being slightly worse than the sequential solution. The split-among-stations and split-over-time methods reduce both metrics by around 35\% and 80\% respectively, despite $n$ being as low as 10. Table~\ref{tab:reltosimtime_2000} shows the different methods, running times for $t>1$ relative to the simulation time.
The highest value is 18.85\% of the simulation time, which is still far below the threshold of 100\% and therefore deemed acceptable.
\begin{figure}[H]
	\centering
	\includegraphics[width=0.55\textwidth]{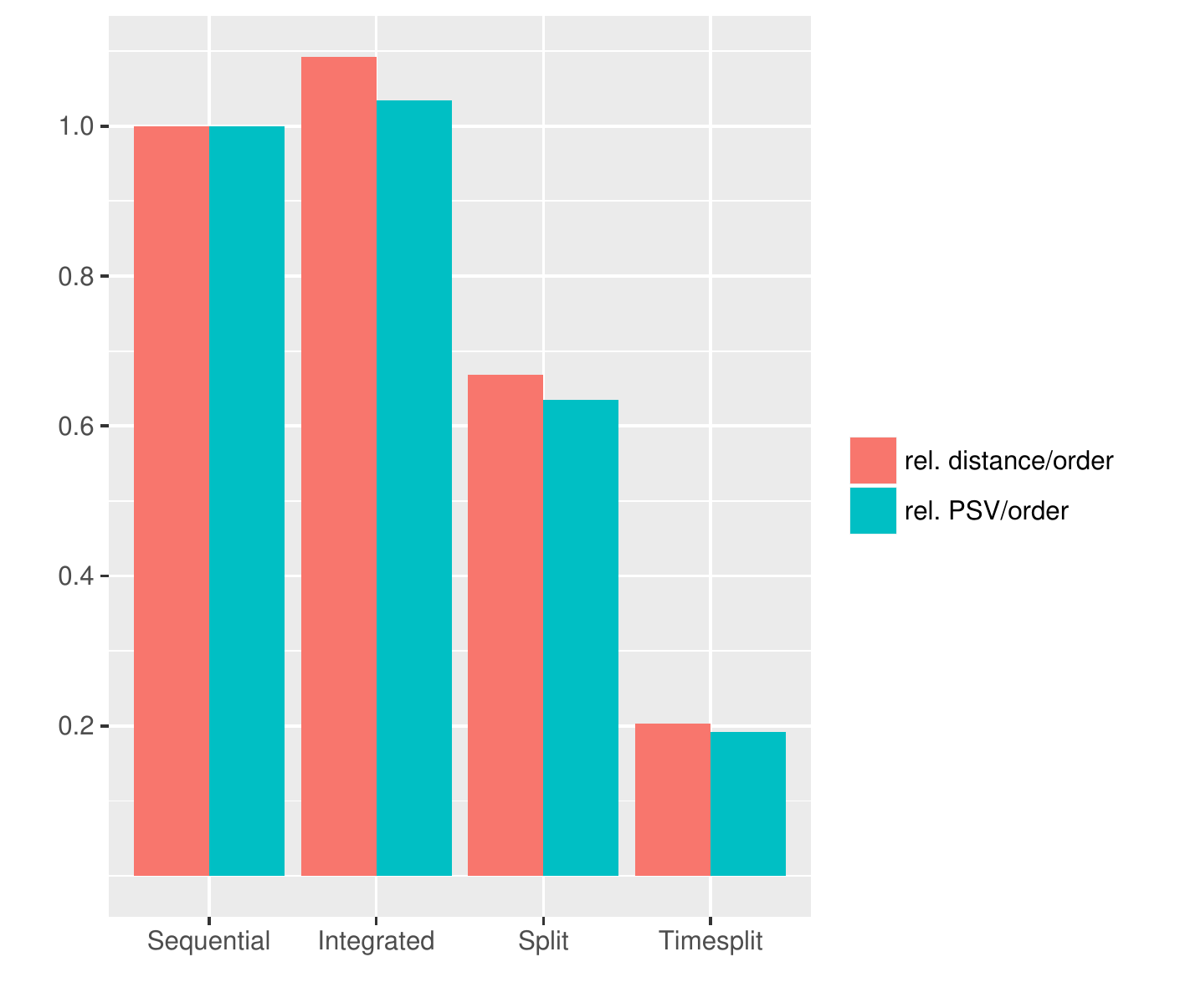}
	\caption{Pod-station visits per order and robot distance per order for the large instance set of 2000 orders.}
	\label{fig:psv_distance_2000_fil10_8bots}
\end{figure}
\begin{table}[H]
	\centering
	\begin{tabular}{ l l l l l}
		\hline
		Instance set & Sequential & Integrated & Split & Timesplit\\
		\hline
		2000 & 0.21\% & 2.14\% & 4.37\% & 18.85\% \\
		\hline
	\end{tabular}
	\captionof{table}{Relative running times compared to the simulation times for $t>1$.}
	\label{tab:reltosimtime_2000}
\end{table}

\subsection{Practical remarks} \label{subsec:practical}
Some assumptions in \Subsecref{assumption} might differ from real-world scenarios. In this section we discuss them from a practical point of view.
\begin{itemize}
	\item Real-world orders: in this paper we use pregenerated instances to test the performance of different approaches. Instead, in a real RMFS, new orders would constantly come in while the optimization algorithms are running. Even completely new SKUs could be stocked during the optimization. To account for this, our simulation RAWSim-O was recently extended by an interface to an ERP system to allow for its use as a robot control system in real warehouses, as described in \cite{Xie.2018}. This new feature of the simulator could also be used in conjunction with the content presented in this paper, to implement the model presented here in a real warehouse and use real instances instead of pregenerated testing instances.
	\item Capacity of packing stations: for the results shown above, we assumed the capacity of packing stations for split orders to be large enough to store all possible split orders. As this capacity might differ from one company to another, a situation where not enough capacity is available to allow for the splitting of all orders is conceivable. In order to consider this situation, we can extend, for example, the model in \Subsecref{splitmodel} with the following parameters, variables and constraints. We need two additional parameters: $C$ for the total capacity of packing stations (in other words: number of available split orders) and $N$ for the number of stations. The binary variable $y_o^l$ is activated if order $o \in O$ is split, while $n^l$ counts the number of currently active split orders from the previous periods. Note that, this number is decreased by one if one split order is picked completely.
	\begin{align}
	& \sum_{s \in \mathcal{S}} y_{os} -e_o = {y}_o, \; \forall o \in \mathcal{O} \tag{\ref{eq:oap-4}.2}\label{eq:oap-44}\\
	& y_o^l \geq e_o/N, \; \forall o \in \mathcal{O}\label{eq:oap-40}\\
	& y_o^l \leq e_o, \; \forall o \in \mathcal{O}\label{eq:oap-41}\\
	& n^l+y_o^l \leq C \; \label{eq:oap-42}\\
	& y_o^l \in \lbrace 0, 1 \rbrace, \; \forall o \in \mathcal{O} \label{eq:oap-43}\\
	& e_o \in \mathbb{Z} \geq 0, \forall o \in \mathcal{O} \label{eq:oap-215}
	\end{align}
	
	In constraint set \eqref{eq:oap-44}, $e_o$ is counted as the number of additional stations to finish picking order $o$. Constraint sets \eqref{eq:oap-40} and \eqref{eq:oap-41} make sure that $y_o^l$ is equal to one if $e_o \ge 0$, while constraint set \eqref{eq:oap-42} makes sure that the number of split orders is less than $C$. Constraint set \eqref{eq:oap-43} defines $y_o^l$ as binary variables for each order $o \in \mathcal{O}$, while constraint set \eqref{eq:oap-215} ensures that $e_o$ can only have non-negative integer values. 
	
	To see the impact that a limited capacity of packing stations has on the solution, we tested the 250-order instance set for a packing station capacity of 10. When applying this limited capacity to the split-among-stations model, pod-station visits were reduced by 35\% compared to the sequential approach, instead of nearly 50\% when packing capacity is not limited. Note that the capacity limit of 10 is an extreme case and the solution is still better than that of the sequential approach. In the real world, higher capacities are possible.
	\item Item capacity: this capacity is used for comparing the sequential and integrated approaches. The capacity of a picking station in the real world is limited by the number of totes, which depends on the size of a tote. There are the following possibilities to apply our approaches without changing the layout of picking stations: 
	\begin{itemize}
		\item using different sizes of totes, with smaller ones for split orders 
		\item using separators in totes for split orders (an example: LocusBot in Figure~\ref{fig:locusBots}).
	\end{itemize}
\begin{figure}[H]
	\centering
	\includegraphics[width=0.25\textwidth]{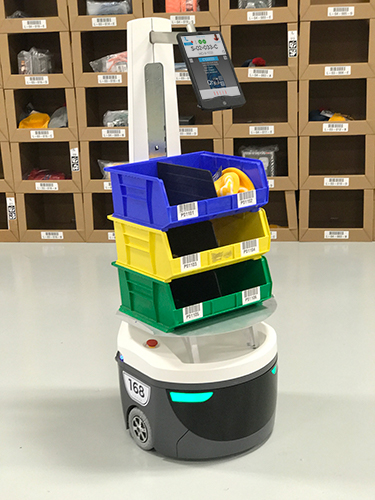}
	\captionof{figure}{An example of using separators in totes (LocusBot).}
	\label{fig:locusBots}
\end{figure}
	\item Reliability of the simulation: as described in \cite{Xie.2018}, the simulation framework RAWSim-O was extended to connect to an ERP system and industry robots. So the optimizers in the simulation, including the newly developed ones in this paper, can be applied directly in real-world scenarios.
\end{itemize}

\section{Conclusions} \label{sec:conclusions}
In an RMFS, the decision on the assignment of orders to stations (POA) affects the throughput of the whole system the most (see \cite{RawSimDecisionRules}). Moreover, the decision on the assignment of pods to stations (PPS) should be made together with POA to get better results (see Example 2 in Section~\ref{subsec:example}). Therefore, we developed novel methods to solve both POA and PPS for multiple stations and make periodic decisions that minimize the number of visits by pods to stations (in short: pod-station visits) to fulfill all customer orders. First, we introduced a new mathematical model to integrate POA and PPS (in short: integrated). Second, we extended the integrated model to allow for split orders. An order is split when not all of its parts are assigned together to a station. Two variations of split orders are considered: split-among-stations (all order lines of a pick order have to be assigned in the same time period, but may be assigned to different picking stations) and split-over-time (order lines of a pick order may be assigned in different time periods and to different picking stations). In the instances we tested, the number of pod-station visits could be reduced by up to 30\% using the integrated model, 50\% using the split-among-stations model and up to 80\% using the split-over-time model compared to the state-of-the-art sequential approach described in \cite{RawSimDecisionRules}.

Additionally, we analyzed the simulation results with regard to three additional performance metrics, namely robot distances, pile-on and order turn-over time. According to our experiments, a reduction in pod-station visits induces a reduction in robot distances and by definition it comes with higher pile-on. The turn-over times for orders could be reduced (except for the split-over-time model), due to the shorter waiting time in the backlog, caused by more efficient order assignments which require the robots to drive less distance and therefore better supply the stations with pods to pick from. 

The running times for our test instances with 50, 150 and 250 orders were acceptable; however, we needed an acceleration method to test large instances with 2000 orders. For instances of that size, it is not practicable to consider all unfulfilled orders in the models at each period; moreover, not all orders are suitable in each period. Therefore, we selected 10 orders according to our prefiltering method in each period to submit to the model. Still, a 30\% or 80\% reduction in pod-station visits could be achieved using the split-among-stations or the split-over-time model respectively, with acceptable running times.

We additionally discussed some assumptions for the mathematical models from a practical point of view, such as the extension of the split-among-stations model to support limited capacity of packing stations.

Since an RMFS is a new type of warehousing system, the concepts
specific to RMFSs have not received much scholarly attention. For example, the implementation of drift spaces (see \cite{Wulfraat.2012}) or priority zones (see \cite{flipse2011altering}). These spaces or zones are located in the proximity of workstations to store pods which might be used in the near future. Determining the optimal size of drift spaces would also be interesting. 

\section*{Acknowledgements}
We would like to thank Dr. Sonja Otten for proofreading the text and the formulas
and her valuable suggestions. We would like to thank the Paderborn Center for Parallel
Computing for providing their clusters for our numerical experiments. This paper is a contribution within the project "Robotic Mobile Fulfillment Systems" that is funded by Beijing Hanning Tech Co., Ltd in China.

\bibliography{orderpicking}

\appendix
\section{Example of the split-over-time approach}
\label{app:example_timesplit}
We have in this example one open position for station 1, while we have two identical orders 1 and 2 in the backlog (in Figure~\ref{fig:t_0_timesplit}). These two orders contain SKUs shown in blue and orange. These two SKUs are located in two different pods, namely pod 1 with the orange SKU and pod 2 with the blue SKU. By allowing orders 1 and 2 to be split into blue and orange parts, the orange ones can firstly be picked from pod 1 at station 1 in period 1, while one position is open at station 2 (see Figure~\ref{fig:t_1_timesplit}). After that, the blue ones are picked from pod 2 at station 2 in period 2 (see Figure~\ref{fig:t_2_timesplit}). This allows both orders to be fulfilled with two pod-station visits.
\begin{figure}[h]
	\begin{subfigure}{0.2\textwidth}
		\hspace{-3cm}
		\includegraphics[width = \textwidth, left]{orders_split.pdf}
	\end{subfigure}
	\newline
	\centering
	\begin{subfigure}[b]{0.3\textwidth}
		\includegraphics[width = \textwidth]{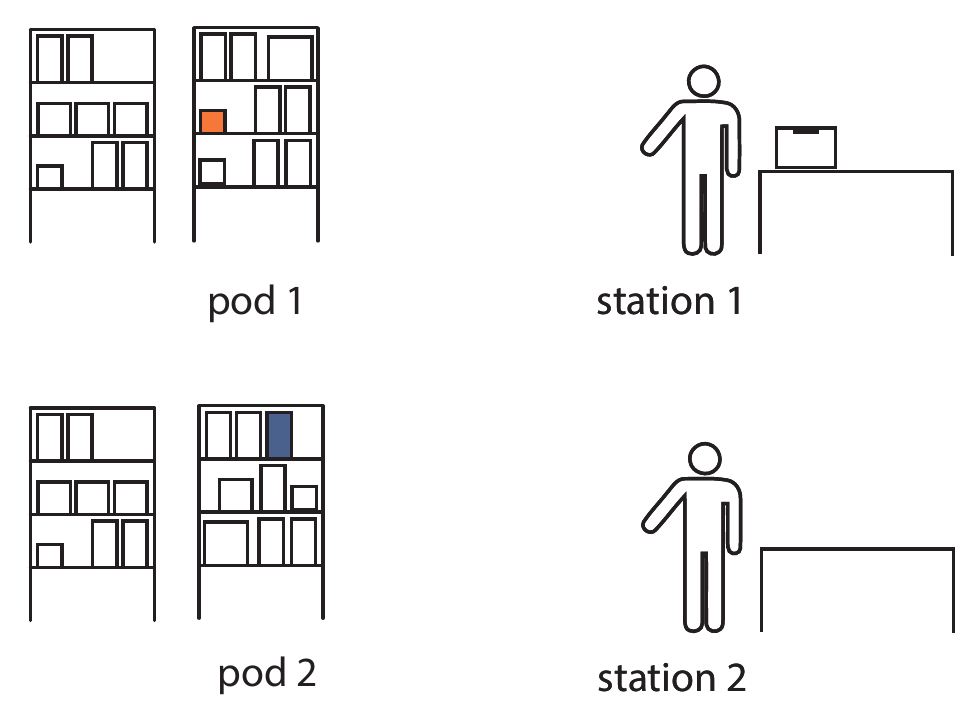}
		\caption{Initial state}	
		\label{fig:t_0_timesplit}
	\end{subfigure}
	\rulesep
	\begin{subfigure}[b]{0.3\textwidth}
		\includegraphics[width = \textwidth]{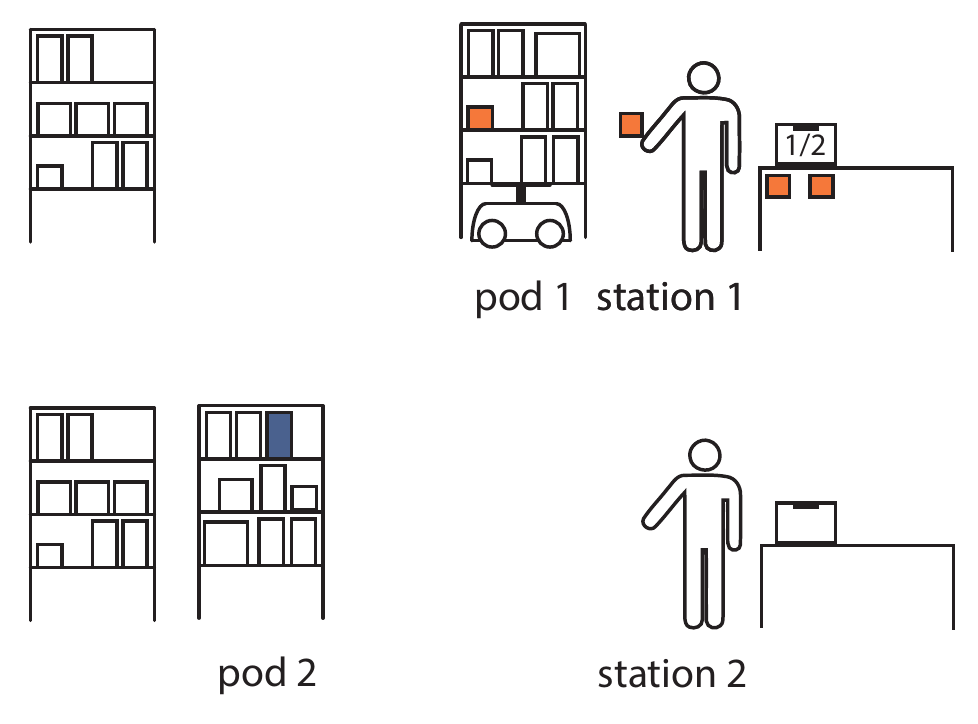}
		\caption{Decision made in period 1}	
		\label{fig:t_1_timesplit}
	\end{subfigure}
	\rulesep
	\begin{subfigure}[b]{0.3\textwidth}
		\includegraphics[width = \textwidth]{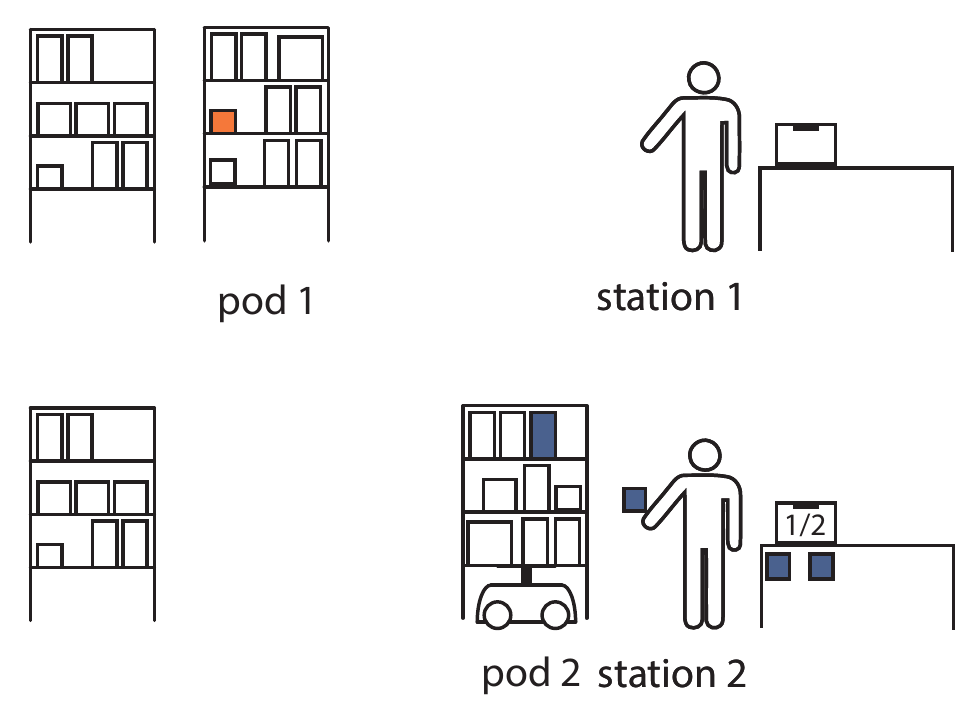}
		\caption{Decision made in period 2}	
		\label{fig:t_2_timesplit}
	\end{subfigure}
	\caption{An example of the split-over-time approach.}
	\label{fig:example_timesplit}
\end{figure}
\section{Calculation example for the capacity of a packing station} \label{app:cal_cap_packing}
As mentioned before, the capacity of a packing station depends on the number of parking shelves and their sizes. According to \cite{Wulfraat.2012}, the typical size of a shelf is 99cm $\times$ 99cm $\times$ 244cm. And we have two common sizes of boxes, namely 25cm $\times$ 17.5cm $\times$ 10cm and 37.5cm $\times$ 30cm $\times$ 13.5cm (see DHL Packsets in sizes S and M in \url{https://www.dhl.de/en/privatkunden/pakete-versenden/verpacken.html}). As shown in Figures \ref{fig:packsets_s} and \ref{fig:packsets_m}, we can store 12 and 6 boxes within a tier of a shelf for the boxes in sizes S and M respectively. Then we assume we store boxes in size S in half of the shelves, and boxes in size M in the other half. By considering spaces for open boxes and usable vertical spaces of a shelf, we assume 100cm of vertical space is available for five tiers of boxes with sizes S, while another 100cm is available for three tiers of boxes with size M. Based on these assumptions, we get to store 78 boxes in total on a shelf. In other words, we can store 78 split orders on a shelf.

\begin{figure}[h]
	\centering
	\begin{subfigure}[b]{0.3\textwidth}
		\includegraphics[width = \textwidth]{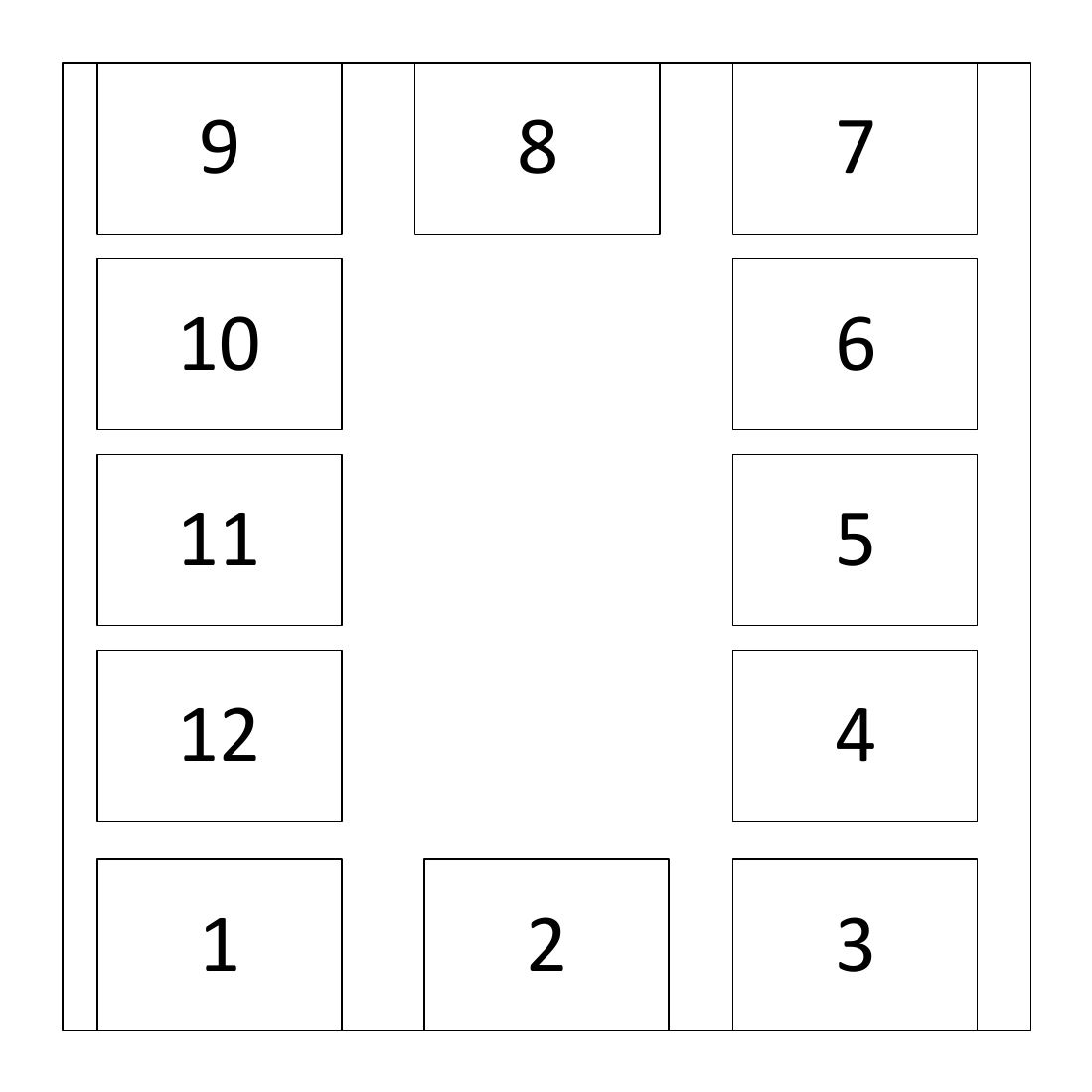}
		\caption{Positions of the boxes in sizes S within a tier of a shelf.}	
		\label{fig:packsets_s}
	\end{subfigure}
	\begin{subfigure}[b]{0.3\textwidth}
		\includegraphics[width = \textwidth]{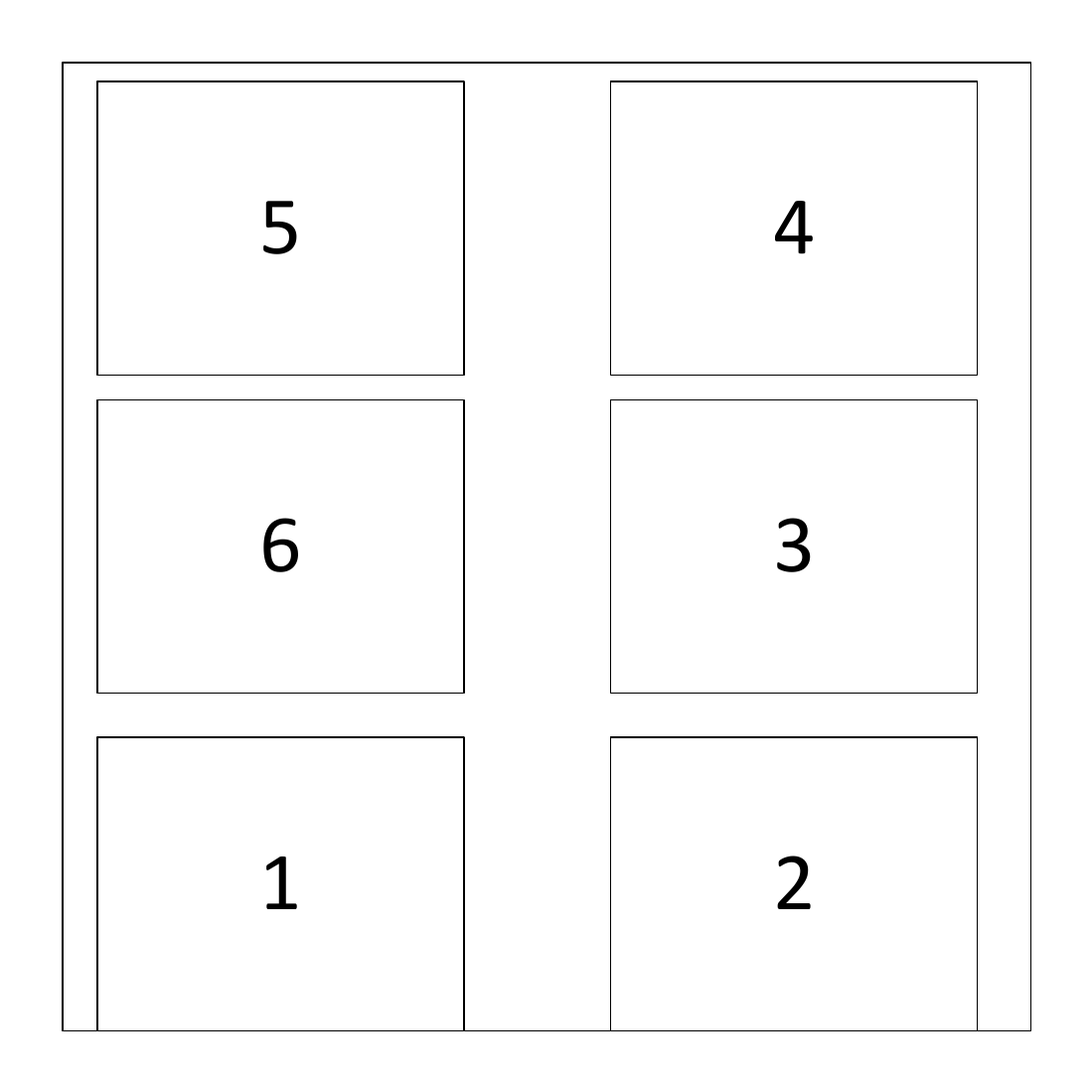}
		\caption{Positions of the boxes in sizes M within a tier of a shelf.}	
		\label{fig:packsets_m}
	\end{subfigure}
\caption{Positions of the boxes in size S (left) and M (right) within a tier of a shelf.}
\end{figure}

\section{Omitted Proofs} \label{app:proofs}
In this section we prove that the models in \Secref{model} always have feasible solutions.
\begin{proof}[Proof of \Propref{existance-of-non-split}]
\newcommand{\ValidOrderSubset}{{\mathcal O'}}
By the assumption about \textit{Maximal order size} in Section \ref{subsec:assumption}, there exists an order $o'$ such that $o'$ has no more items than capacity $C_{s'}$ at some station $s'\in \mathcal{S}$. Let us select this order and the corresponding pods for our solution.

 Formally this means:
Set $x_{ps} := 1$ for all $p \in \mathcal{P}_s$ as required by  \eqref{eq:oap-10}.
	Set $y_{o'} := 1$ and $y_{o} := 0$ otherwise. Set $y_{o's'}:=1$ and set $y_{os}:=0$ otherwise.
	Then \eqref{eq:oap-4} holds.
	Set $y_{io's'}:=1$ for all $i \in \mathcal{I}_{o'}$ and $y_{ios}:=0$ otherwise.
	Then  \eqref{eq:oap-3} holds. By the choice of $o'$ and $s'$ constraint \eqref{eq:oap-5} holds too.
	Finally, we set $x_{ps'}:=1$ for all $p \in \SetOfPodsBySku{i}$ with $i\in \mathcal{I}_{o'}$, then
	\eqref{eq:oap-9} is fulfilled. Thus we have constructed a feasible solution.
\end{proof}
\begin{proof}[Proof of \Propref{existance-of-station-split}]  \label{proof:proof_split_stations}
We show that every solution of the integrated model also solves the 
	split-among-stations model.
	Let $(x'_{ps})_{p \in \mathcal{P}, s \in \mathcal{S}}$, $(y'_{o})_{{o \in \mathcal O}}$, $(y'_{os})_{o \in \mathcal{O}, s \in \mathcal{S}}$,  and $(y'_{ios})_{o \in \mathcal{O}, i_o \in \mathcal{I}_o,  s \in \mathcal{S}}$ be a solution of the integrated model.
	From \eqref{eq:oap-3} follows \eqref{eq:oap-23}.
	Set   $e'_o:=0$ for all  $o \in \mathcal{I}_o$, then from \eqref{eq:oap-4} follow \eqref{eq:oap-24}, \eqref{eq:oap-2} and \eqref{eq:oap-215}.
	Substitute \eqref{eq:oap-3} into \eqref{eq:oap-4}, then \eqref{eq:oap-26} follows.
	If we sum \eqref{eq:oap-4} on both sides of the equation over $i\in \mathcal{I}_o$ we get $\sum_{i_o \in \mathcal{I}_o}y'_{ios} = \sum_{i_o \in \mathcal{I}_o} y'_{os} \geq y'_{os}$ and therefore $\eqref{eq:oap-27}$ holds.

From the first part of this proof and \Propref{existance-of-non-split}, it follows that there is a feasible solution for the split-among-stations model.

Finally, we have the same objective function \eqref{eq:oap-1}. Because the split-among-stations problem is an optimization problem, its optimal solution is either\\ $(x'_{ps})_{p \in \mathcal{P}, s \in \mathcal{S}}$, $(y'_{o})_{{o \in \mathcal O}}$, $(y'_{os})_{o \in \mathcal{O}, s \in \mathcal{S}}$,  $(y'_{ios})_{o \in \mathcal{O}, i_o \in \mathcal{I}_o,  s \in \mathcal{S}}$ or better.
\end{proof}
\begin{proof}[Proof of \Propref{existance-of-station-time-split}] \label{proof:proof_split_time}
The proof is analogous to the proof of \Propref{existance-of-station-split}.
	Let $(x'_{ps})_{p \in \mathcal{P}, s \in \mathcal{S}}$, $(y'_{o})_{{o \in \mathcal O}}$, $(y'_{os})_{o \in \mathcal{O}, s \in \mathcal{S}}$,  and $(y'_{ios})_{o \in \mathcal{O}, i_o \in \mathcal{I}_o,  s \in \mathcal{S}}$, $(e'_o)_{o \in \mathcal{O}}$, be a solution of the split-among-stations model.
	Set all $y'^b_{i_o}:=0$, then \eqref{eq:oap-36} holds.
	Furthermore, we have the same objective function
	\eqref{eq:oap-21}. With the same argumentation as in the proof of \Propref{existance-of-station-split}, we conclude all the statements of \Propref{existance-of-station-time-split}.
\end{proof}
\section{Additional parameters in the simulation} \label{app:restpara}
\begin{table}[H]
	\centering
	\begin{tabular}{ l l }
		\hline
		Parameter & Value\\
		\hline
		Robot acceleration/deceleration &  $1 \frac{m}{s^2}$\\
		Robot maximum velocity & $1.5 \frac{m}{s}$\\
		Time needed for a full turn of a robot & $2.5s$\\
		Time needed for lifting and storing a pod & $2.2s$\\
		Time needed for picking a unit & $7s$\\
		Time needed for handling a unit at picking station & $13s$\\
		\hline
	\end{tabular}
	\captionof{table}{Parameters of robot movement and time for picking units.}
	\label{tab:Otherparams}
\end{table}

\end{document}